\documentclass[letterpaper, 10 pt, conference]{ieeeconf}  

\IEEEoverridecommandlockouts                             

\usepackage{array}
\usepackage{url}

\newcommand{\bigO}[1]{\mathcal{O}\left(#1\right)}

\newcommand{\model}{\textsc{GameOpt}\xspace}
\newcommand{\modelNew}{\textsc{GameOpt+}\xspace}

\newcommand{\auctionname}{\textsc{Game}\xspace}
\newcommand{\optname}{\textsc{Opt}\xspace}

\newcommand{\vts}[1]{\lvert #1 \rvert}

\makeatletter
\newcommand\footnoteref[1]{\protected@xdef\@thefnmark{\ref{#1}}\@footnotemark}
\makeatother

\newcommand{\shorteq}{%
  \settowidth{\@tempdima}{-}
  \resizebox{\@tempdima}{\height}{=}%
}

\usepackage{amssymb,fge}

\usepackage[table]{xcolor}
\definecolor{sg}{HTML}{00ff7f}
\definecolor{lb}{HTML}{b0f5ef}
\definecolor{lg}{HTML}{9bfaa8}

\usepackage{amsmath, amssymb}
\usepackage{pifont}
\usepackage{subcaption}
\usepackage[utf8]{inputenc}
\usepackage[english]{babel}
\usepackage[table]{xcolor}
\usepackage[linesnumbered,ruled,vlined]{algorithm2e}
\usepackage[font=small,belowskip=-1.5pt]{caption} 
\usepackage{anyfontsize}
\usepackage{array}
\usepackage{graphicx}
\usepackage{amsfonts}
\usepackage[11pt]{moresize}
\usepackage{bm}
\usepackage{soul}
\usepackage{hhline}
\usepackage{multirow, makecell}
\usepackage{mathrsfs}
\usepackage{float}
\usepackage{booktabs,wrapfig}

\usepackage{amsthm}
\usepackage{color}
\usepackage{transparent}
\usepackage{url}
\usepackage{footmisc}
\usepackage{setspace}
\usepackage[colorlinks=true, citecolor=magenta]{hyperref}
\usepackage{mathtools}
\theoremstyle{plain}

\newtheorem{theorem}{Theorem}[section]

\begin{document}

\title{\modelNew: Improving Fuel Efficiency in Unregulated Heterogeneous Traffic Intersections via Optimal Multi-agent Cooperative Control}



\author{Nilesh Suriyarachchi$^{1}$, Rohan Chandra$^{2}$, Arya Anantula$^{2}$, John S. Baras$^{3}$ and Dinesh Manocha$^{3,4}$
\thanks{$^{1}$Nokia Bell Labs, USA, $^{2}$Computer Science Department, The University of Texas at Austin, Texas, USA, $^{3}$Electrical and Computer Engineering Department, University of Maryland, College Park, Maryland, USA, $^{4}$Computer Science Department, University of Maryland, College Park, Maryland, USA. Email: \tt\small\{rchandra@utexas.edu\}}
}



\maketitle

\begin{abstract}

Better fuel efficiency leads to better financial security as well as a cleaner environment. We propose a novel approach for improving fuel efficiency in unstructured and unregulated traffic environments. Existing intelligent transportation solutions for improving fuel efficiency, however, apply only to traffic intersections with sparse traffic or traffic where drivers obey the regulations, or both. We propose \modelNew, a novel hybrid approach for cooperative intersection control in dynamic, multi-lane, unsignalized intersections. \modelNew is a hybrid solution that combines an auction mechanism and an optimization-based trajectory planner. It generates a priority entrance sequence for each agent and computes velocity controls in real-time, taking less than $10$ milliseconds even in high-density traffic with over $10,000$ vehicles per hour. Compared to fully optimization-based methods, it operates $100\times$ faster while ensuring fairness, safety, and efficiency. Tested on the SUMO simulator, our algorithm improves throughput by at least $25\%$, reduces the time to reach the goal by at least $70\%$, and decreases fuel consumption by $50\%$ compared to auction-based and signaled approaches using traffic lights and stop signs. \modelNew is also unaffected by unbalanced traffic inflows, whereas some of the other baselines encountered a decrease in performance in unbalanced traffic inflow environments.

\end{abstract}

\section{Introduction}
\label{sec: introduction}

Efficiency and safety are two of the core pillars defining a successful transportation system. With regard to road infrastructure, safety targets involve achieving a system capable of operating with zero collisions. Efficiency targets are more complex and can be broadly broken into time and energy. Efficiency in time involves getting from point to point in the least amount of time, while efficiency in energy focuses on using the least amount of fuel. Improvements in both time and energy efficiency can provide significant economic gains in terms of reducing fuel used and minimizing time wasted in traffic. In the United States in 2022, the average driver spent a fuel cost of $\$546$ and a time spent cost of $\$869$ with $51$ hours spent on the road\footnote{\href{https://inrix.com/scorecard/}{https://inrix.com/scorecard/}}. Even a $5\%$ improvement in efficiency would then provide the US with at least a $\$16.5$ $bn$ economic benefit. 

In addition to the savings in fuel costs and time saved, it is also important to consider the environmental gains. Fuel efficiency also results in a reduction of greenhouse gas emissions which would help combat the ever-present threat of global warming and climate change \cite{yoro2020co2, mikhaylov2020global}. Additionally, the resulting reduction in air pollution would have a major impact on health conditions, especially in urban areas where air quality is often poor. With the dwindling supply of global crude oil, reducing overall fuel consumption will also help buy additional time for renewable sources to become more efficient and accessible \cite{kabeyi2022sustainable}. 

\begin{figure} [t] 
\centering
\includegraphics[width=\columnwidth]{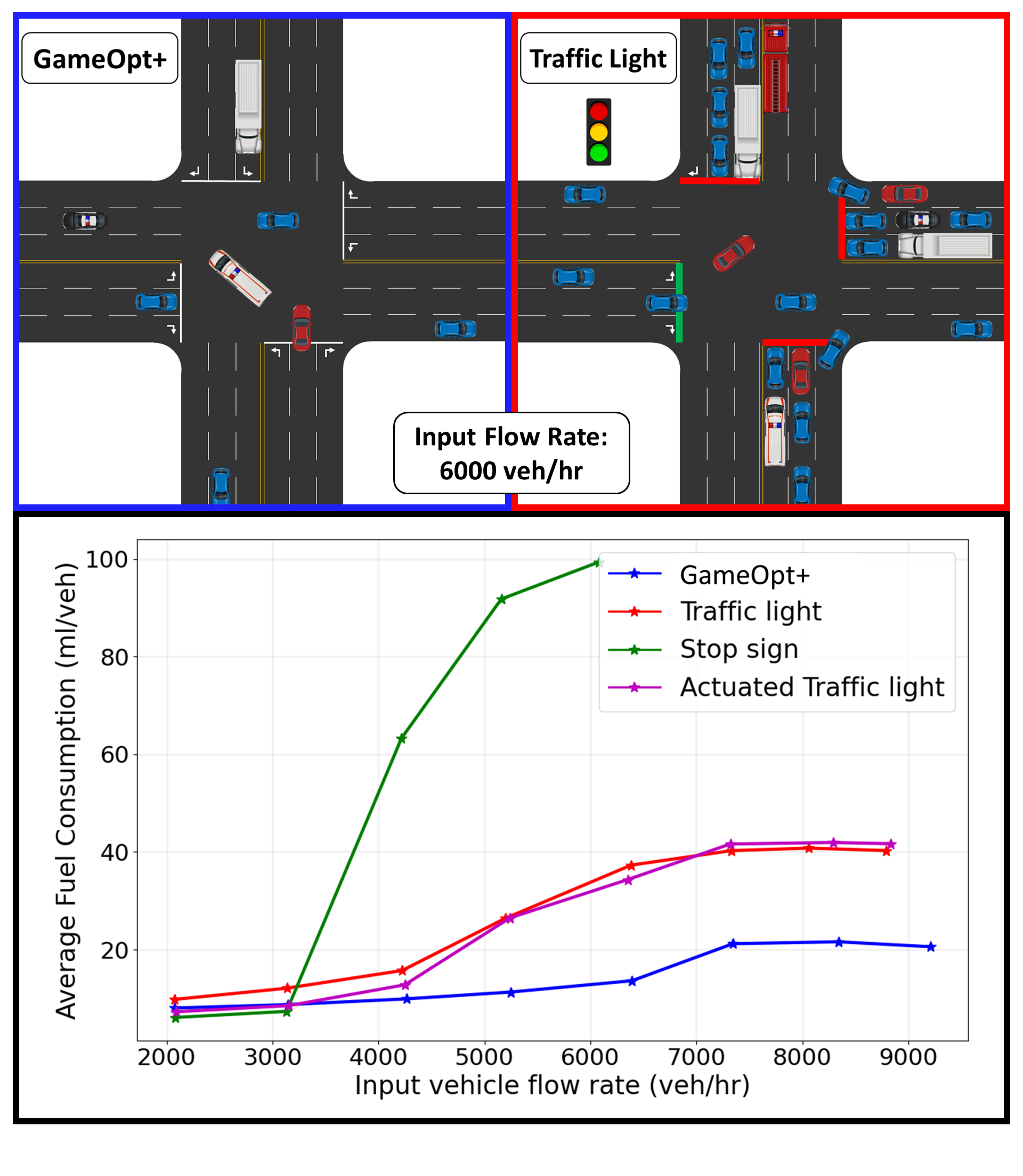}
\caption{\textbf{\modelNew:} We present a new approach for improving fuel efficiency via optimal real-time planning and control in dynamic, heterogeneous, multi-agent unsignalized intersections. We show that at identical input traffic flow levels, our approach results in the most fuel efficient traffic flow.}
\label{fig: cover}
\end{figure}
Fuel efficiency can be achieved in multiple ways such as optimizing the path of travel, minimizing acceleration and braking steps, and through vehicle design involving engine efficiency, aerodynamics, and energy regeneration \cite{romero2024strategies}. The global effort to push for hybrid and electric vehicles has also helped push the efficiency bar forward. While the task of improving vehicular design generally lies in the hands of automobile manufacturers, improvements to vehicle trajectories often require infrastructure and regulatory support. Since the majority of fuel efficiency gains can generally be expected from the second category it is essential that road infrastructure systems become more intelligent and capable.

Intelligent transportation systems (ITS)~\cite{chandra2019densepeds, chandra2019roadtrack, chandra2022towards, guan2022m3detr} utilize sensors, cameras, and communication systems to gather information about traffic conditions and optimize traffic flow by adjusting traffic signals to help reduce delays and improve fuel efficiency \cite{guerrero2018sensor}. While these sensors can measure factors such as vehicle speed, traffic volume, and congestion, the installation of sufficient sensors is often cost-prohibitive~\cite{chandra2023meteor}. In terms of control capabilities, traditionally transportation systems have been limited to simple methods such as traffic signals and variable speed limits \cite{qadri2020state}. However, the advent of Connected Autonomous Vehicles (CAVs) is about to revolutionize the sensing and control capabilities of an intelligent transportation system. CAVs have a multitude of onboard sensors such as cameras and rangefinders which allow them to accurately track the state of surrounding traffic. They also possess advanced communication systems which allow them to share this information with each other as well as with infrastructure facilities \cite{liu2023systematic}. Finally, in terms of control, CAVs can accurately plan trajectories and execute consistent maneuvers with high levels of predictability. They can also follow commands or guidelines set by the infrastructure control facilities in an ITS \cite{liu2023systematic}. As such, by leveraging the capabilities of CAVs, modern transportation systems would possess a significantly higher degree of sensing and control capability. These increased capabilities allow the control systems of ITS to pursue more sophisticated and better-performing control methods which leads to improved efficiency and safety.

Physical bottlenecks such as traffic intersections and highway merge junctions as well as phenomena such as traffic shock waves play a major role in reduced safety and increased congestion on road networks \cite{eom2020traffic, junhua2016modeling}. Of these, often the most challenging scenario is the unregulated traffic intersection, \textit{i.e.} intersections with no stop sign or traffic light. Unregulated intersections are responsible for a significant portion of road accidents, with $40\%$ of all crashes, $50\%$ of serious collisions, and $20\%$ of fatalities occurring there~\cite{grembek2018introducing}. To address this issue, the concept of connected autonomous vehicles (CAVs) equipped with vehicle-to-infrastructure (V2I) communication has emerged. By introducing additional sensing and actuation points, these CAVs enable the development of new algorithms to tackle unregulated intersections. Successful applications of this approach include handling complex traffic bottlenecks like highway merging and traffic shock waves\cite{basic,Nilesh2021merge,Nilesh2021shockwave}. Recent research has explored leveraging cooperation among CAVs to achieve safe, fair, and efficient intersection control, drawing on various fields such as game theory, auctions, optimization, and deep learning. However, no single solution has demonstrated the ability to satisfy all four key considerations: safety, efficiency, fairness, and real-time operation. Safety involves preventing collisions, efficiency focuses on maximizing intersection capacity, fairness requires equitable treatment of all agents, and real-time computation is crucial for adapting to dynamic changes in this fast-paced environment.


The problem of controlling intersections for multiple agents can be divided into two main phases. The planning phase focuses on determining the optimal entrance sequence, the order in which vehicles should enter the intersection. The control phase involves generating safe trajectories for all vehicles to follow in order to adhere to the selected sequence. The resulting target trajectories are then transmitted to each vehicle's local controller for execution. However, selecting the sequence becomes increasingly challenging as the number of vehicles grows, leading to an exponential increase in possibilities. This complexity is further compounded in multi-lane scenarios where multiple vehicles can enter the intersection simultaneously without conflicts to maximize efficiency. Practical intersection scenarios are also \textit{dynamic}, meaning that the number of vehicles requesting to cross the intersection changes over time. Finally, we must also consider the impact of heterogeneous vehicles. Different types of vehicles differ in their dynamics and requirements.


\subsection{Related Work}
\label{subsec: related_work}

\subsubsection{Solutions for Improving Fuel Efficiency}

Several approaches have been aimed at efficiency improvements for autonomous vehicles. By communicating traffic signal information to vehicles, an optimization-based control algorithm can be utilized to predict an efficient velocity trajectory \cite{Pred_cruise}. The objectives were selected to improve efficiency by reducing idle time at stoplights, arriving at green lights with minimal braking, and cruising at the desired speed.
Another optimization-based technique for intersection management focused on using cooperative adaptive cruise control to minimize delays at intersections \cite{iCACC}. The approach demonstrated that the reduction of delays at intersections led to significant savings in fuel consumption.

In relatively recent approaches, platoons (a group with a leading vehicle and following vehicles) are used to increase driving stability and reduce air resistance between vehicles which minimizes wasted gasoline. Additionally, reinforcement learning is utilized to calculate optimal start-to-end paths with the storage and computation load on network edges \cite{platoon_rl}. In creating the routes, the vehicles consider the fuel efficiency of a route as a key factor.

\subsubsection{Towards Navigating Unregulated Traffic Intersections}

Previous studies have explored the use of auction mechanisms for navigating unsignalized intersections in real time while ensuring fairness and feasibility~\cite{chandra2022gameplan, carlino2013auction}. However, these mechanisms are limited to \textit{static} scenarios with a fixed number of participating agents, resulting in subpar efficiency and throughput when the number of vehicles is variable. Incentive-compatible auctions, such as the mechanism proposed by Sayin et al.~\cite{sayin2018information}, allocate turns to agents based on their distance from the intersection and the number of passengers in their vehicles. Likewise, Carlino et al.~\cite{carlino2013auction} and Rey et al.~\cite{rey2021online} propose a similar approach but employ a monetary-based bidding strategy. Buckman et al.~\cite{gt6} incorporate a driver behavior model~\cite{schwarting2019social} into a first-in, first-out framework to account for human social preferences. More recently, Chandra et al.~\cite{chandra2022gameplan, cmetric, chandra2019graphrqi, chandra2021using} introduced an auction mechanism based on the driving behavior of the agents.


Autonomous vehicle navigation has also seen the application of game-theoretic approaches~\cite{li2020game, tian2020game, chandra2022gameplan}. Li et al.~\cite{li2020game} employ a stackelberg game where one agent assumes the role of the leader, following the first-out (FIFO) principle, while the other agent acts as a follower. In contrast, Tian et al.~\cite{tian2020game} adopt a recursive $k-$level approach, deriving strategies at each level from previous levels. However, in these approaches, all agents, except the ego agent at the first level, are treated as static entities.
 

Although optimization-based methods offer safety guarantees by computing optimal trajectories, their practical implementation in real-time is hindered by high computational costs. Existing approaches address this limitation by making unrealistic assumptions. Rios-Torres \textit{et al.} \cite{basic,Malik2019unconstrained} propose a real-time solution to the optimization problem using Hamiltonian analysis, but this approach does not extend well to fully constrained problems. Computation is also limited to scenarios with fewer vehicles on single-lane roads, excluding turning at intersections, to enable near real-time calculations~\cite{Bian2020opt,Pei2021spaceopt,Seyed2018opt}. In our previous work \cite{Nilesh2021merge}, Suriyarachchi \textit{et al.} applied a hybrid control approach combining rules and optimization to achieve real-time performance in the context of highway merging. However, relying solely on rule-based sequence selection leads to sub-optimal solutions.


Deep learning techniques, including reinforcement learning~\cite{capasso2021end, mavrogiannis2022b, mavrogiannis2020b, sathyamoorthy2020densecavoid} and recurrent neural networks~\cite{roh2020multimodal, chandra2019traphic, chandra2019robusttp, chandra2020forecasting}, have been utilized to train planning and control policies for agents approaching intersections. However, these policies often struggle to generalize effectively across different environments and lack guarantees regarding safety and fairness. Additionally, most learning methods face challenges when applied to dynamic scenarios involving multiple agents for planning and control.


In our previous work, we developed \model~\cite{gameopt}, a hybrid planning and control algorithm for navigating unsignalized dynamic intersections by combining auctions with optimal control. Key features of \modelNew involved the computation of a game-theoretic-based safe, fair, and efficient intersection merge sequence with a corresponding optimization-based control strategy that could function in real-time. The system was shown to handle multiple input lanes ad varying levels of traffic flow.
Despite its evident advantages over prior methods for navigation at unregulated intersections, \modelNew did not account for heterogeneous vehicles. Different types of vehicles differ in their fuel capacity and requirements.

\subsection{Main Contributions:}

We propose a novel approach, called \modelNew, that uses game theory and optimization to improve fuel efficiency.

\begin{itemize}
    \item \modelNew simulates heterogeneous vehicles that have different lengths, velocities, and accelerations.
    \item \modelNew also incorporates driver preference/aggressiveness that influences speed and fuel usage.
\end{itemize}
\section{Intersection Control Formulation}
\label{sec: problem_formulation}

Optimal vehicle flow at unregulated intersections maximizes fuel efficiency. We define optimal vehicle flow via the following characteristics:

\begin{enumerate}
    \item \textit{Fairness}: The velocities, $v_i$, determined should maximize utility for all agents and encourage incentive compatibility.
    \item \textit{Safety}: The trajectories produced by the optimizer should ensure collision-free paths.
    \item \textit{Efficiency}: The algorithm as a whole should optimize throughput, time-to-goal, and fuel efficiency.
    \item The planning and control system should not rely on accessing the objective or utility functions of other agents.
    \item The planning and control system should not assume a dynamics model; instead, it should calculate dynamics for each agent based on a simple forward motion model.
\end{enumerate}

To initiate the process, the initial stage entails establishing the physical attributes of an unsignalized intersection, along with comprehending the dynamics and control mechanisms of the involved CAVs.


\subsection{Modeling the Physical Intersection}
\label{subsec: physical_intersection}

An unsignalized intersection is characterized as a four-way crossing with no traffic signals or right-of-way regulations and can be either single-lane or multi-lane. In our work, we present an example of a single-lane intersection in Fig. \ref{fig: zones}. To facilitate control and coordination, we define a control zone of length $L_c$ along each arm of the intersection. Within this control zone, vehicles can exchange state information and receive actuation commands through V2I communication protocols. For the purpose of this research, we assume ideal conditions with zero transmission delay and a flawless V2I communication channel. The center of the intersection, known as the conflict zone, poses a risk of collisions. The number of vehicles present in the control zone of roads $0,1,2,3$ is represented by $n_1,n_2,n_3,n_4$, respectively, with the total number of vehicles in the control zone being $n=n_1$+$n_2$+$n_3$+$n_4$.


\begin{figure} [h!] 
\centering
\includegraphics[width=.9\columnwidth]{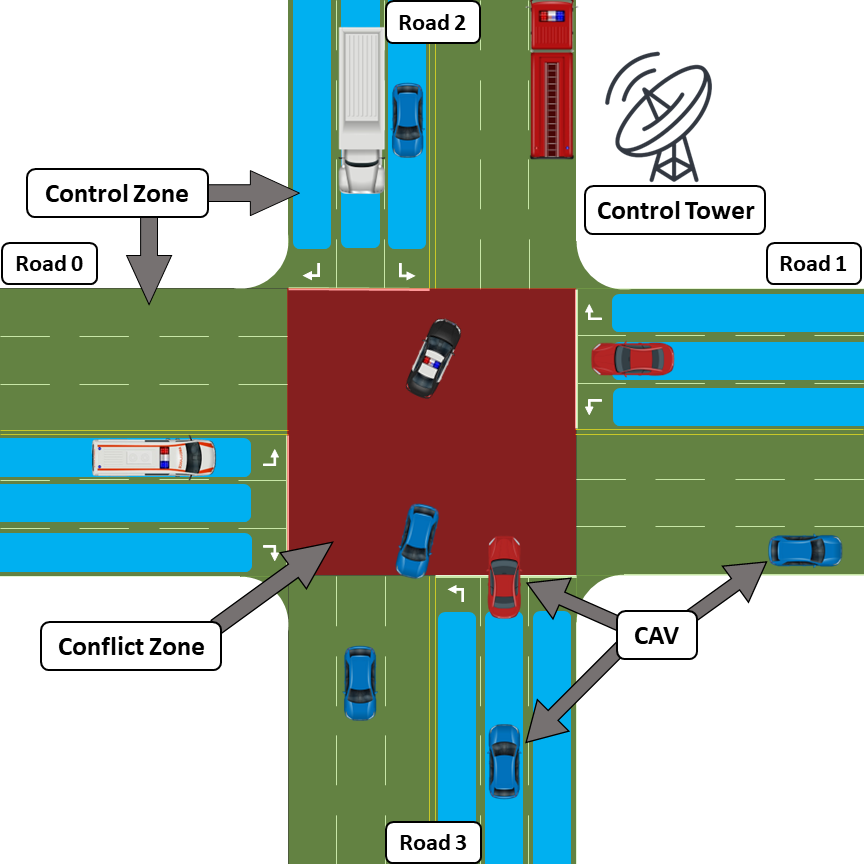}
\caption{\textbf{Regions of interest:} We highlight the control zone in which vehicles communicate with the control tower and the conflict zone in which vehicles cross over to their desired target road.}
\label{fig: zones}
\vspace{-7pt}
\end{figure}

The vehicles operating within the intersection have the autonomy to make their own decisions, including choosing to turn left, turn right, or proceed straight. It is possible to adjust the input flow rate of vehicles in each control zone of the intersection, as well as the proportion of vehicles opting to turn left, turn right, or go straight. The input flow rate is varied in the range of $2,000$ to $10,000$ vehicles per hour, with the generation of vehicles following a Poisson distribution.



\subsection{Vehicle Dynamics and modeling}
\label{sec: control_var_dynamics}
In real-world scenarios, vehicle dynamics are hard to model accurately due to their non-linearity. Standard practice in control theory suggests the use of a local controller, $z_i$, that produces low-level controls for the longitudinal and lateral motion of a vehicle according to the following equation, 

%
\begin{equation}
\label{eq:eq3}
    \frac{\delta s_i}{\delta t} = \mathcal{F}(t,s_i,z_i).
\end{equation}
\begin{figure*}[t] 
\centering
\includegraphics[width=\textwidth]{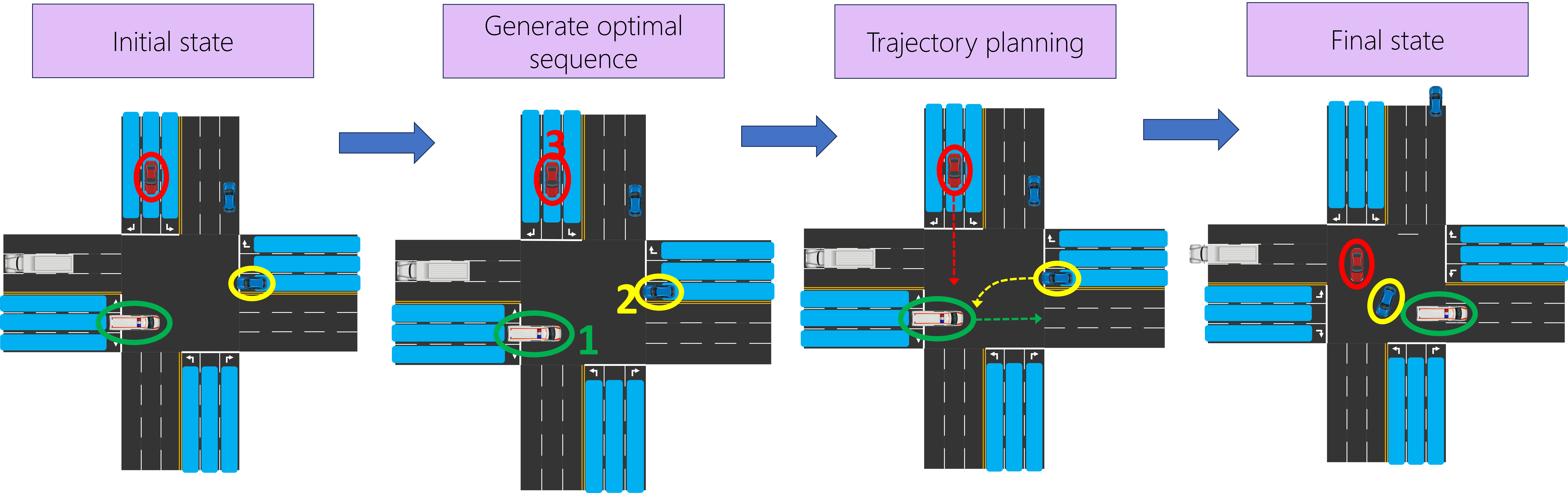}
\caption{\textbf{\modelNew overview:} Our approach begins by reading in the positions and velocities of all agents in the control zone (blue). Our approach is hybrid; in the planning phase, \modelNew collect the bids from every agent and generate an optimal priority sequence (Section~\ref{subsec: priority_order}). In the controls phase, we use an optimization-based trajectory planner to compute the optimal velocity for each agent that satisfies the priority order while simultaneously guaranteeing safety and real-time performance (Section~\ref{sec:optimization}).}
\label{fig: flow_diagram}
\vspace{-15pt}
\end{figure*}

\noindent Here, we represent $\mathcal{F}$ as the unknown non-linear dynamic function. $s_i$ denotes the distance of the agent $i$ from the intersection measured in the Frenet coordinate system. 

\noindent \textbf{Dynamics:} Since $\mathcal{F}$ is generally unknown, we define the high-level dynamics by the following velocity control scheme:
\begin{equation}
\begin{aligned}
\frac{\delta s_i}{\delta t}  &= v_i \\
v_i(t) &= \bm{u}_i(t) 
\end{aligned}
\label{eq:eq1}
\end{equation}
\noindent for $i\in\{1,\ldots,n\}$. 
 

\noindent \textbf{Control variables:} The control input, $\bm{u}_i(t)$, represents the command velocity for vehicle $i$. We compute the optimal value for $\bm{u}_i(t)$ with respect to safety, efficiency and reachability guarantees described in the following sections.

\noindent \textbf{State vector:} In addition to $s_i(t)$ and $v_i(t)$, we assign 
an indicator variable $l^c_i(t)\in\{0,1,\ldots,L_n-1\}$, which 
signifies which lane the vehicle is currently in, with $L_n$ representing the number of lanes on each road. We also obtain the length $l_i$, and maximum acceleration $a^{max}_i$ and deceleration $a^{min}_i$ capabilities of each CAV $i$.
In a scenario of multi-arm multi-lane intersections, we need to identify which direction each vehicle intends to take through the intersection. Therefore, we allocate each vehicles in the control zone into a group $G_i$ based on its desired trajectory (origin road and intended direction) as shown in section \ref{subsec: conflict_resolution}, in order to resolve conflicts inside the intersection.
%

\noindent \textbf{Handling heterogeneous vehicles:} When dealing with heterogeneous multi-class traffic we find that different classes of vehicles can have significantly different goals. Some classes may prefer to prioritize reaching their destination and speed while others would prefer to maintain fixed speeds and not perform too many acceleration tasks in order to save fuel. For example, an ambulance cares more about speed maximization, while a large semi-trailer truck would rather minimize its overall acceleration tasks due to the relatively large fuel costs associated with these actions. In addition to vehicle class we find that other factors such as total waiting time, current speed and individual vehicle intent or aggressiveness plays a major role in what features each car prioritizes. In order to achieve a more optimal control solution it is necessary that each vehicle should be treated differently based on its specific needs. 
As such we assign a vehicle type $c_i$ which depends on the class of vehicle and a driver preference (aggressiveness) value $d_i$, which takes a value from $0$ to $1$ indicating a vehicles bias towards either driving as fast as possible $(\text{towards }1)$ or reducing fuel usage $(\text{towards }0)$.
We also include the waiting time $w_i(t)$ indicating how long it has been since the vehicle entered the control zone.

The CAVs are thereby completely defined by the state vectors: 
\begin{equation}
\label{eq:eq2go_ext}
    \bm{\Phi}_i(t) = [s_i(t), v_i(t), l^c_i(t), w_i(t), c_i, d_i, G_i, l_i, a^{max}_i, a^{min}_i]^\top
\end{equation}

\noindent for $i\in\{1,\ldots,n\}$. We also define a parameter tuple $M_s = \{M_{sr},M_{sl}\}$, which represent the safety margins needed to prevent rear-end and lateral collisions.

\subsection{Dynamic Traffic Intersection Planning \& Control Problem}
\label{sec: initial_formulation}
The dynamic traffic intersection problem involves computing suitable, collision-free, continuous trajectories for all vehicles approaching an intersection to achieve efficient and fair traffic flow. 
This problem can be formulated as an optimal control problem, with the objective of computing the optimal command velocities $\bm{u}_i$ for each CAV $i\in\{1,\ldots,n\}$, resulting in the minimization of the maximum time taken by each CAV to cross the intersection:

\begin{equation}
\begin{aligned}
&~~ \min_{\{\bm{u}_i\}} ~~~ \max_i ~ t_f^i \\
&s.t. \quad\quad \mathcal{C}(\{\bm{\Phi}_i\},\bar v, \mathbb{S}(M_s))
\end{aligned}
\label{eq:opt_time}
\end{equation}
\noindent where $t_f^i$ is the intersection crossing time of vehicle $i$. $\mathcal{C}$ is a set of constraints, discussed in Section~\ref{sec:optimization}, that depend on the state vectors $\{\bm{\Phi}_i\}$ of the CAVs, and safety requirements $\mathbb{S}(M_s)$ which prevent collisions between vehicles.

Equation~\ref{eq:opt_time} is, however, intractable to optimize directly because it implicitly assumes that agents are navigating the intersection according to the optimal priority order, $q^*$. Equation~\ref{eq:opt_time} can be equivalently solved by optimizing jointly over sequences, $q$, and the control velocities, $\bm{u}_i$, needed to implement the sequence.

\begin{equation}
\begin{aligned}
&~~ \min_{\{\bm{u}_i,q\}} ~~~ \sum_{i=1}^n 
    \lambda (\bm{u}_i-\bar v)^2 + (1-\lambda) (\bm{u}_i-v_i)^2 \\
&s.t. \quad\quad \mathcal{C}(\{\bm{\Phi}_i\},\bar v,\mathbb{S}(M_s),q)
\end{aligned}
\label{eq:opt_velocity}
\end{equation}
\noindent where $\bar v$ is the speed limit. We solve this optimization problem in Section~\ref{sec:optimization}.\\

\noindent\textbf{Intractability of Equation~(\ref{eq:opt_velocity}):} The optimal control velocities, $\bm{u}_i$, depend on the sequence order $q$ of the vehicles entering the conflict zone. The naive solution is to solve a mixed-integer optimization problem for every possible entrance sequences $q$ by performing an exhaustive search. This process involves two sub-problems---generating a sequence $q$ followed by computing the optimal velocity commands $\bm{u}_i$ for that sequence. The total number of possible sequences $q$ grows exponentially with the number of vehicles in the control zone which would take the form $\bigO{\frac{(n_1+n_2+n_3+n_4)!}{n_1!n_2!n_3!n_4!}}$, where $n$=$n_1$+$n_2$+$n_3$+$n_4$ is the number of vehicles in the control zone. Therefore, current methods to solve Equation~\ref{eq:opt_velocity} are intractable.

We propose the use of auctions to obtain an optimal priority sequence order $q^{*}$ in $\bigO{n\log n}$ time, significantly reducing the combinatorial runtime complexity of the mixed-integer formulation. We run the optimization over $\bm{u}_i$ for the selected optimal sequence $q^{*}$ which can be solved in real time.

\subsection{Sponsored Search Auctions}
\label{subsec: SSA_background}

Sponsored search auctions (SSAs) are a game-theoretic mechanism that are used extensively in internet search engines for the purpose of internet advertising~\cite{roughgarden2016twenty}. In an SSA, there are $K$ items to be allocated among $n$ agents. Each agent $a_i$ has a private valuation $\bm{v}_i$ and submits a bid $\bm{b}_i$ to receive at most one item of value $\bm{\alpha}_i$. A strategy is defined as an $n$ dimensional vector, $\bm{b} = (\bm{b}_i \cup \bm{b}_{-i})$, representing the bids made by every agent. $\bm{b}_{-i})$ denotes the bids made by all agents except $a_i$. Furthermore, let $\bm{b}_1 > \bm{b}_2 > \ldots > \bm{b}_K$ and $\bm{\alpha}_1 > \bm{\alpha}_2 > \ldots > \bm{\alpha}_K$. The allocation rule is that the agent with the $i^\textrm{th}$ highest bid is allocated the $i^\textrm{th}$ most valuable item, $\bm{\alpha}_i$. The utility $\bm{u}_i$~\cite{roughgarden2016twenty} incurred by $a_i$ is given as follows,

\begin{equation}
 \bm{u}_i (\bm{b}_i) =  \bm{v}_i \bm{\alpha}_i - \sum_{j=i}^k \bm{b}_{j+1} \left( \bm{\alpha}_j - \bm{\alpha}_{j+1} \right).
 \label{eq: utility_template}
\end{equation}

\noindent In the equation above, the quantity on the left represents the total utility for $a_i$ which is equal to value of the allocated goods $\bm{\alpha}_i$ minus a payment term. The first term on the right is the value of the item obtained by $a_i$. The second term on the right is the payment made by $a_i$ as a function of bids $\bm{b}_{j > i}$ and their allocated item values $\bm{\alpha}_j$. We refer the reader to Chapter $3$ in~\cite{roughgarden2016twenty} for a derivation and detailed analysis of Equation~\ref{eq: utility_template}.

In our approach, we re-cast Equation~\ref{eq: utility_template} through the lens of a human driver. More specifically, the term $\bm{v}_i \bm{\alpha}_i$ denotes the time reward gained by driver $a_i$ by moving on her turn. The payment term represents a notion of risk~\cite{wang2020game} associated with moving on that turn. It follows that an allocation of a conservative agent to a later turn (smaller $\bm{\alpha}$) also presents the lowest risk and vice-versa.

\section{\modelNew}
\label{sec: gameopt}

The overall approach is depicted in Figure~\ref{fig: flow_diagram}, showcasing the flow of computation. Initially, the positions, velocities, and initial bids from all agents within the control zone are acquired through V2I communication (Section~\ref{sec: problem_formulation}). Subsequently, in the planning phase, this information, together with the processed bids, is utilized to generate an optimal priority order or sequence, dictating the order of agent entry into the intersection (Section~\ref{subsec: priority_order}). Moving into the control phase, our approach employs an optimization-based trajectory planner to determine the optimal velocity for each agent (Section~\ref{sec:optimization}). This optimization process ensures adherence to the priority order, while simultaneously upholding safety and real-time performance guarantees.


\subsection{\auctionname: Selecting a Priority Order, $q$}
\label{subsec: priority_order}

Choosing an optimal ordering for agents to navigate unsignalized traffic scenarios is equivalent to allocating each agent a turn in which they would cross the intersection. Such an allocation depends on the incentives of the agents which, in many cases, are not known apriori. Auctions model the incentives of agents in unsignalized traffic scenarios using an optimal combination of bidding, allocation, and payment strategies. In the rest of this section, we present the auction framework for generating an optimal priority order, followed by an analysis of its fairness.

The Sponsored Search Auction (SSA), as detailed in Section~\ref{subsec: SSA_background}, presents an optimal mechanism for generating a priority order for agents. This algorithm allows the agent with the highest priority bid to navigate the scenario first, followed by the agent with the next highest priority, and so on. With a polynomial runtime, dominated by sorting the agents' bids, the SSA is an efficient solution~\cite{clrs}. In GameOpt~\cite{gameopt}, agents' bids were based on a single dimension - the current kinematics of each agent. However, in this work, we propose a more advanced bidding language where agents can submit multi-dimensional bids, expressing their preferences using various features or currencies such as driver behavior-based~\cite{chandra2022gameplan}, distance-based~\cite{gt6}, and time-based~\cite{carlino2013auction}.

Formally, we represent the set of bidding currencies as $\mathscr{F}^m$, where $f \in \mathscr{F}^m$ denotes an $m$-dimensional feature vector comprising bidding currency values. We further define $\Omega^m \subseteq \mathbb{R}^m$ as the space of coefficient vectors $\omega \in \Omega^m$. A master bid, $\bm{b}_i: \mathscr{F}^m \times \Omega^m \longrightarrow \mathbb{R}$, is then generated and passed on to the SSA. To address multi-agent dynamic traffic, we implement an overflow strategy where, if an agent with a higher priority is behind an agent with a lower priority, the former transfers a portion of their bid to the latter, similar to the Social Forces model.

Typically, SSAs employ a single bidding currency in their frameworks, reducing the space of bidding strategies to $\mathscr{F}^1$. Although some recent studies~\cite{multikeyword1, multikeyword2} propose extensions that allow bidding on multiple slots or keywords, these auctions still utilize a single currency (bidding over keywords). However, utilizing a single bidding currency poses several challenges, primarily that if an agent has a specific budget $B$ for multiple keywords and spends $90\%$ of $B$ towards a high-profile keyword, they may not have sufficient funds left to bid on other keywords.

In contrast, multi-currency bidding strategies theoretically allow agents to realize multiple budgets and thus, expand their power to express their preferences. However, multi-currency bidding strategies are challenging to implement in practice, as many currencies lack constraints, and the dimension of the space $\mathscr{F}^m$ grows exponentially with $m$.

In this work, we propose a method for handling multi-currency bidding strategies. Specifically, we consider the features \emph{time to intersection} ($T^i_I$), \emph{distance to intersection} ($D^i_I$), \emph{waiting time} ($W^i_I$), and \emph{driver preference or assertiveness} ($A^i_I$). Thus the full feature (bidding currency) set is defined as $f= [T^i_I, D^i_I, W^i_I, A^i_I]$, each of which constitutes a continuous feature, with $m=4$. The master bid $\bm{b}_i$ is then computed as follows:

\begin{equation}
\bm{b}_i = \omega^\top f
\end{equation}

Which using the coefficients $\omega_k$ where $k\in\{1,\dots,m\}$ can also be written as,

\begin{equation}
\bm{b}_i = \omega_1*T^i_I + \omega_2*D^i_I + \omega_3*W^i_I + \omega_4*A^i_I
\label{eq: master_bid}
\end{equation}

\noindent Next, we describe the equations for computing each bidding feature. Recall that $s_i(t)$ measures the distance of agent $i$ from the intersection, and we use $\tau_i(t) = s_i(t)/v_i(t)$ to denote the time in which $i$ shall reach the intersection based on its current velocity and $s_i(t)$. The bidding capacity of vehicles should be higher when they are closer to the intersection in terms of time and distance. Therefore, the features $T^i_I$ and $D^i_I$ for agent $i$ are computed as follows:
\begin{equation}
T^i_I =  (c_1 - \tau_i(t))
\end{equation}
\begin{equation}
D^i_I =  (c_2 -  s_i(t))
\end{equation}

where $c_1$ and $c_2$ are constants representing maximum bounds on control zone length and time to intersection respectively. Next, agents specify a preference for waiting in queue for a long time, by the cumulative time it has been waiting in the control zone $w_i(t)$ (available in vehicle state vector in equation (\ref{eq:eq2go_ext})). Feature $W^i_I$ is then set as $W^i_I  =  w_i(t)$. Finally, agents also express their behavior through a driver preference assertiveness feature $A^i_I$. For these features vehicles with higher waiting time and higher preference have a higher bidding capacity. The computation of the preference feature $A^i_I$ additionally incorporates the effects of heterogeneous vehicle types. Therefore, different classes of vehicles will have varying bounds on their preference bidding capacity $A^i_I$. For example, on a scale of 1 to 10, emergency vehicles would have the highest bidding range of \{7 to 10\} and passenger cars would have a lower range such as \{1 to 5\}. Therefore the maximum bidding capacity within the driver preference assertiveness feature $A^i_I$ is dependant on the vehicle class. This allows vehicles with higher priority needs to be given preference over other vehicle classes during the sequence order generation phase. Here, the actual vehicle class $c_i$ and individual driver preference $d_i$ is available in the vehicle state vector in equation (\ref{eq:eq2go_ext}).

By using the combined master bid obtained for each CAV $i$ as shown in equation (\ref{eq: master_bid}), a sponsored search auction is conducted to generate the final feasible merging sequence $q$. Here, vehicles with higher combined bids would get the opportunity to pass through the intersection before vehicles with lower bids. As the bidding system is directly linked to the dynamics of the system, the sequence generated by the auction ensures that vehicles higher up in the sequence have the dynamic capability to pass the intersection before vehicles lower down in the sequence. Certain other features such as the overflow strategy discussed also play a roll in ensuring the overall feasibility of the generated sequence.

\subsubsection*{\ul{Generation of multiple feasible sequences}}
\label{subsec: multi_seq_gen}
Thus far we described the method of generating a single feasible sequence $q$, using the set of bidding currencies as $\mathscr{F}^m$, and the space of coefficient vectors $\omega \in \Omega^m$. The choice of coefficients $\omega_k$ where $k\in\{1,\dots,m\}$, is a tuning parameter available to the central controller which can be changed in order to generate multiple feasible sequences. It is difficult to tell which sequence would be truly optimal in terms of maximizing vehicle throughput while minimizing fuel consumption at this stage. Therefore it is the responsibility of the \auctionname subsystem (section \label{subsec: priority_order}) of the control system, to generate a suitable collection of candidate feasible sequences to be tested by the optimization \optname subsystem (section \label{sec:optimization}) for optimality. Thus, each of the candidate sequences $q$, generated in \auctionname, will be fed into the optimization \optname subsystem in parallel to find the best sequence $q^*$, whose corresponding target command velocities would be sent to each CAV in the control zone.

\subsubsection*{\ul{Fairness Analysis of $q$}}
\label{subsec: optimality}

 Looking at the game-theoretic aspect, fairness is achieved when no agent has an incentive to ``cheat'' meaning that each agent's dominant strategy is to bid their true valuation, denoted as $\bm{\zeta}_i$. This principle, known as incentive compatibility, aims to encourage traffic agents to truthfully bid their actual value ($\bm{b}_i = \bm{\zeta}_i$) as the optimal strategy~\cite{roughgarden2016twenty}. Chandra et al.~\cite{chandra2022gameplan} demonstrated that Single-Stage Auctions (SSAs) implemented in static intersections with one agent per arm are incentive-compatible. We expand the theoretical framework of static SSAs to dynamic intersections involving multiple vehicles, considering the occurrence of \textit{overflow} when a lower-priority vehicle obstructs a higher-priority one.


In~\cite{chandra2022gameplan}, the authors show that the sponsored search auction is game-theoretically optimal for static intersections consisting of $1$ vehicle on each arm of the intersection. In this work, we extend the proof for dynamic traffic intersections where we must take into account three factors--$(i)$ multiple vehicles on each arm of the intersection and $(ii)$ lower priority vehicles blocking higher priority vehicles, and $(iii)$ providing waiting time rewards to agents.

\begin{theorem}
\textbf{Incentive compatibility for dynamic intersections:} For each agent $i$ for $i=1,2,\ldots,n$ and $n = n_1 + n_2 + n_3 + n_4$ where $n_j$ represents the number of vehicles on the $j^\textrm{th}$ arm, bidding $\bm{b}_i = \bm{\zeta}_i$ is the dominant strategy. 
\label{thm: incentive_compatibility}
\end{theorem}

\begin{proof}

\begin{enumerate}
    \item Use SSA to compute an ordering over all agents. But this can generate false bids overbids and underbids.
    
    \item For example, an agent may be assigned to go second on account of a higher priority vehicle.
    
    \item But the latter vehicle may not even conflict with the first vehicle. Conversely, agent $i$ might be asked to go first, risking a collision.
    
    \item Therefore, it is important to identify conflicting and non-conflicting lane groups. Run the SSA in conflicting groups only.
    
    \item At this point, we can say that since we have optimized conflicting groups and non-conflicting groups do not collide anyways, the resulting auction is incentive-compatible
\end{enumerate}

\end{proof}

The next desired property in a fair auction is welfare maximization~\cite{roughgarden2016twenty} which maximizes the total time reward (Section~\ref{subsec: SSA_background}) earned by every agent. 
We show that SSAs maximize social welfare for dynamic intersections as well.

\begin{theorem}
\textbf{Welfare maximization for dynamic intersections:} Social welfare of an auction is defined as $\sum_i \bm{\zeta}_i\bm{\alpha}_i$ for each agent $i$ for $i=1,2,\ldots,n$ and $n = n_1 + n_2 + n_3 + n_4$ where $n_j$ represents the number of vehicles on the $j^\textrm{th}$ arm. Bidding $\bm{b}_i = \bm{\zeta}_i$ maximizes social welfare for every agent. 
\label{thm: Welfare_Maximizing}
\end{theorem}
\begin{proof}

We can show that welfare is maximum for each set of conflicting lane group. We need to show that social welfare of the system is maximized when the welfare of each individual conflicting lane group is maximized. Our proof is based on induction. Base case ($n=1$): for $n=1$, this reduces to the static auction case~\cite{chandra2022gameplan}. Hypothesis: Suppose the current system consists of $n>1$ conflicting groups and that the social welfare of this system is maximized. Then we want to show that the addition of $\vts{n+1}$ agents (belonging to the $(n+1)^\textrm{th}$) conflicting group also maximizes social welfare of the system. Now, welfare of system with $n$ conflicting groups is

\[\sum_{i=1}^n\sum_{j=1}^{\vts{i}}\zeta^i_j\alpha^i_j\]

\noindent From the hypothesis, we know that the above is maximum. Now if the $(n+1)^\textrm{th}$ conflicting group is added, then we can maximize that using the static SSA auction~\cite{chandra2022gameplan}. The resulting sum is a sum of two optimized terms, so the resulting welfare of the system remains maximum.
\end{proof}

Finally, we propose a novel strategy to address overflow by transferring a portion of the higher priority agent's bid to the lower priority agent i. We denote such a transfer by $(\hat a, \hat b = a \xrightarrow{c} b)$, where $\hat a = a-c, \hat b = b+c$. More formally, 

\begin{theorem}
\textbf{Overflow prevention:} In a current SSA, suppose there exists $ (i,j) \in [n_k] \times [n_k], k=1,2,3,4$, such that $\bm{\zeta}_i > \bm{\zeta}_j$ and $s_i[t] > s_j[t]$. Let $ \bm{\hat \zeta}_i, \bm{\hat \zeta}_j = (\bm{\zeta}_i \xrightarrow{q} \bm{\zeta}_j)$. If 
\[ q < \bm{\zeta}_i \left( 1-\frac{\alpha_{i+m}}{\alpha_i}\right) - \sum_{s=i}^{i+m-1}\bm{\zeta}_{s+1}\left( \frac{\alpha_s - \alpha_{s+1}}{\alpha_i} \right),\] 

\noindent the new SSA with $\bm{\hat \zeta}_i, \bm{\hat \zeta}_j$ as the new priority values for $i,j$ is incentive compatible.
\label{thm: overflow}
\end{theorem}
\begin{proof}

The difference in utility for agent $i$ by transferring some of her valuation to agent $j$ is given by the following equation. Here, we assume that by executing this transfer, agent $i$ gained a jump of $m$:

\begin{equation}
\begin{split}
    \Delta u = (v_k - q)\alpha_k - \big[& v_{k+1}(\alpha_k-\alpha_{k+1} + v_{k+2}(\alpha_{k+1}-\alpha_{k+2} + \ldots\\
    &+ v_{k+m}(\alpha_{k+m-1}-\alpha_{k+m}) \big] - v_k\alpha_{k+m}\\
\end{split}
\end{equation}

Rearranging the terms, $\Delta u > 0$ if,

\begin{equation}
\begin{split}
    &v_k(\alpha_k - \alpha_{k+m}) - \sum_{j=k}^{k+m-1}v_{j+1}(\alpha_j - \alpha_{j+1}) > q\alpha_k\\
    \implies&q < v_k(1 - \frac{\alpha_{k+m}}{\alpha_{k}}) - \frac{1}{\alpha_k}\sum_{j=k}^{k+m-1}v_{j+1}(\alpha_j - \alpha_{j+1}) \\
\end{split}
\end{equation}

\end{proof}

As a final remark, Carlino et al.~\cite{carlino2013auction} show that multiplying a bid by a waiting time reward does not change the incentive compatibility of the auction. 
We defer the proofs of Theorems~\ref{thm: incentive_compatibility},~\ref{thm: Welfare_Maximizing}, and~\ref{thm: overflow} to \cite{gameoptPage}.

\subsection{Multi-lane Intersection Planning and Control }
\label{subsec: conflict_resolution}

An essential feature of our approach is its effectiveness in efficiently managing traffic flow in multi-lane intersections. However, transitioning from a single-lane to a multi-lane setting presents unique challenges. Assumptions that hold true in single lanes, such as the presence of only one vehicle at a time in the conflict zone, no longer hold in multi-lane scenarios. To address this, we introduce a classification system that groups all vehicles in the control zone ($G_i$) based on their intended trajectory, including the origin road and their intention. Each group is denoted as \textit{x-y: x=road\_num, y=intention} (with intention values assigned as: \textit{0=right\_turn, 1=go\_straight, 2=left\_turn}). For instance, a vehicle from road 1 intending to make a left turn would belong to group 1-2. In Fig. \ref{fig: conflict_groups}, we illustrate the possible trajectories for vehicles in each lane and the corresponding labeling used to assign groups based on these trajectories, indicated on each lane.

\begin{figure}[t]
\begin{subfigure}[h]{.9\columnwidth}
\vspace{5pt}
\centering
    \includegraphics[width = 0.9\textwidth]{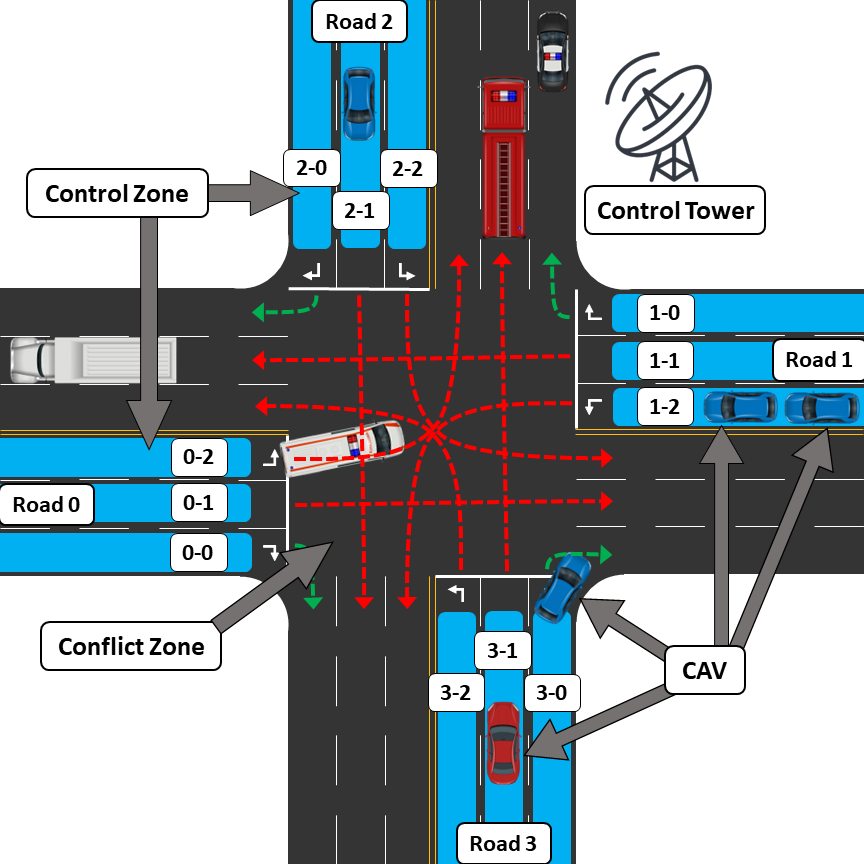}
        \caption{Multi-lane intersection structure with lane-based groups}
        \label{fig: conflict_groups}
\end{subfigure}
  \begin{subfigure}[h]{\columnwidth}
  \vspace{15pt}
    \centering
\resizebox{.6\columnwidth}{!}{
\begin{tabular}{cc}
\toprule[1.5pt]
{\textbf{Lane group}} & {\textbf{No-conflict group set}} \\ 
\midrule
0-1 & 0-2, 1-1, 2-2 \\
0-2 & 0-1, 1-2, 3-1 \\
1-1 & 0-1, 1-2, 3-2 \\
1-2 & 0-2, 1-1, 2-1 \\
2-1 & 1-2, 2-2, 3-1 \\
2-2 & 0-1, 2-1, 3-2 \\
3-1 & 0-2, 2-1, 3-2 \\
3-2 & 1-1, 2-2, 3-1 \\
\bottomrule[1.5pt]
\end{tabular}
}
    \caption{Non-conflicting trajectory groups at intersection}
    \label{table: grouping}
    \end{subfigure}
\label{fig: conflict_figure}
\caption{\textbf{Conflict handling at multi-lane intersections}. We propose a novel strategy for conflict handling using vehicle grouping based on non-colliding trajectories.}
\vspace{-15pt}
  \end{figure}

We observe that vehicle groups intending to make right turns at the intersection (groups 0-0, 1-0, 2-0, 3-0) do not create conflicts with any other groups. As a result, these groups are granted unrestricted entry into the intersection, and the optimization constraints outlined in \ref{sec:optimization} reflect this allowance. However, it is crucial for the remaining lane groups to determine which other groups can also enter the intersection simultaneously without causing conflicts. We refer to this collection as the "non-conflicting group set" for each corresponding main lane group, as demonstrated in Table \ref{table: grouping}. For instance, let's consider the main lane group 0-1. The groups with trajectories that do not clash with this main group include 0-2, 1-1, and 2-2, along with all the right turn groups (0-0, 1-0, 2-0, 3-0). Essentially, this implies that vehicles belonging to the non-conflicting group set of a main lane group can enter the intersection concurrently with a vehicle from the main lane group itself. It's important to note that the right turn at intersection groups (0-0, 1-0, 2-0, 3-0) are inherently considered non-conflicting and are automatically included in the non-conflicting groups list for all vehicle groups. The determination of non-conflicting groups is predicated on the assumption that vehicles approaching the conflict zone have already completed any necessary lane changes strategically.


\subsection{\optname: Computing Optimal Velocities} 
\label{sec:optimization}

\noindent\textbf{The Optimization Problem:}

Based on the priority sequence $q$ from Section~\ref{subsec: priority_order}, the optimal command velocities, $\bm{u}_i$, for each vehicle $i$ in the control zone are computed by optimizing Equation~\ref{eq:opt_velocity} for the fixed sequence $q$. The objective function restated:
\begin{equation}
\mathcal{J}(\bm{u}_i|q) = \min_{\{\bm{u}_i\}} \sum_{i=1}^n \lambda (\bm{u}_i-\bar v)^2 + (1-\lambda) (\bm{u}_i-v_i)^2
\label{eq:J}
\end{equation}

\noindent The first quadratic term of Equation~\ref{eq:J} prevents $\bm{u}_i$ from deviating from the speed limit  $\bar v$ (results in improved throughput). Thus each vehicle will attempt to maximize its velocity within limits. The second quadratic term minimizes the control effort applied to optimize fuel efficiency. This term minimizes the selection of control goals resulting in strong acceleration and deceleration tasks. The parameter $\lambda$ offers a trade off between these two competing terms and provides a method to tune and prioritize one term over the other. 

As described in section \ref{sec: control_var_dynamics}, when we are dealing with heterogeneous traffic it becomes increasingly important to treat individual vehicles differently depending on their needs and priorities. There are two main types of priorities we consider in this research. The first type is the Speed Priority ($P^s_i$), which represents each vehicles desire to maximize its individual velocity. The second type is the Speed Variation Priority ($P^v_i$), which represents each vehicles desire to minimize variations to its current speed, thereby minimizing acceleration tasks and overall fuel consumption. 
Based on the vehicle class $c_i$ and vehicle preference $d_i$, a priority assignment node computes these two types of priorities ($P^s_i, P^v_i$) for each vehicle $i$ in the control zone. Each vehicle depending on its class $c_i$ has upper and lower bounds on its priority values. The exact value assigned within these class bounds takes into account each vehicle's individual preference $d_i$. While the $\lambda$ parameter in Equation~(\ref{eq:J}) provides a way to trade off between speed maximization and fuel savings it only does so in a global sense. In order for each individual vehicle to adjust its preference towards speed maximization or fuel savings, we introduce the two priority terms ($P^s_i, P^v_i$) into the objective function resulting in the following updated objective function:
\begin{equation}
\mathcal{J}(\bm{u}_i|q) = \min_{\{\bm{u}_i\}} \sum_{i=1}^n \lambda P^s_i (\bm{u}_i-\bar v)^2 + (1-\lambda)  P^v_i (\bm{u}_i-v_i)^2
\label{eq:J2}
\end{equation}
Here, each individual vehicle $i$, can affect the bias towards either throughput or fuel savings depending on its own individual preferences captured in its two priority values. This allows for a more optimal solution which can cater to the individual needs of each vehicle.
Below we discuss constraints on safety and the influence of $q$ on Equation~(\ref{eq:J2}).\\

\noindent \textbf{Constraints on Safety and Compliance with $\mathbf{q}$:} 

We consider both longitudinal (rear-end) collisions in the control zone and collisions in the conflict zone. Rear-end collisions along each of the lanes in the control zone are prevented by ensuring all vehicles in the same lane maintain a safe distance between each other. Formally,

\begin{equation} 
\Big | s_j^k(t+1)-s_{j\prime}^k(t+1)\Big | \geq l_{j} + M_{sr}
\label{eq: longitudinal_collision}
\end{equation}

\noindent for all $k$ lanes in the control zone and for all agents $j,j\prime$ with $j\neq j\prime$ in the $k^\textrm{th}$ lane. $l_j$ represents the length of the agent $j$ and $M_{sr}$ is a safety margin added to maintain a safe gap between vehicles and to account for imperfections in sensing and actuation. We compute the future position $s_i(t+1)$ using,

\begin{equation}
\label{eq:eq_cons_3_1}
    s_i(t+1) = s_i(t) - \Delta t \left (\frac{v_i(t)+\bm{u}_i(t)}{2}\right)
\end{equation}

\noindent Here, $\Delta t$ is defined as the planning time step of the controller.

Next, we also need to ensure that no collisions occur inside the conflict zone. This task reduces to ensuring no vehicles belonging to conflicting lane groups (\ref{subsec: conflict_resolution}) enter the intersection simultaneously. Therefore, for each vehicle $i$ in the sequence $q$, paired with a vehicle $j$ in a conflicting group where $i$ has higher priority over $j$, we add the constraint,

\begin{equation}
        t_c^i + \frac{(l_i + M_{sl})}{\bm{u}_i} \leq t_c^{j}
    \label{eq: lateral_collision}
\end{equation}

\noindent where $t_c^i= \frac{s_i(t+1)}{\bm{u}_i}$ represents the time at which $i$ would arrive at the intersection. The above constraint ensures that command velocities $\bm{u}_i$ are chosen such that vehicles proceed through the conflict zone in the order of the sequence under consideration $q$. Note that this constraint is applied to every conflicting pair of vehicles; non-conflicting vehicles are allowed to enter the intersection at the same time, increasing the efficiency of the system. The safety margin $M_{sl}$ corresponds to the width of the intersection and prevents lateral collisions occurring when conflicting vehicles cross the intersection. 

The safety constraints with respect to ordering $q$ in Equations~(\ref{eq: longitudinal_collision}) and~(\ref{eq: lateral_collision}) result in an intractable mixed-integer quadratic programming (MIQP) optimization problem. This problem is further compounded due to lane indexing, which introduces a new variable indicating the agent's lane. We remove the dependency over the lane index by separating the constraints for each lane and using the multi-lane conflict resolution scheme introduced in~\ref{subsec: conflict_resolution}. This results in simplification of Equations~(\ref{eq: longitudinal_collision}) and~(\ref{eq: lateral_collision}), reformulating the original MIQP as a quadratic programming (QP) optimization problem; this QP can be solved in real-time. More formally, Equation~(\ref{eq: longitudinal_collision}) can be simplified as,

\begin{equation}
\label{eq: long_linear}
    \begin{aligned}
        \bm{u}_{j}^k-\bm{u}_{j+1}^k \geq & \left(v_{j+1}^k-v_{j}^k\right) \\
        & + \frac{2}{\Delta t}\left(s_{j}^k-s_{j+1}^k+l_{j}+M_{sr}\right) 
    \end{aligned}
\end{equation}

\noindent for all $k$ lanes in the control zone and for all $j,j+1$ consecutive vehicles belonging to lane $k$. And Equation~(\ref{eq: lateral_collision}) can be simplified as,

\begin{equation}
\label{eq: lateral_linear}
        \bm{u}_{j}\left(s_i-\frac{\Delta t}{2}.v_i+l_i+M_{sl}\right) 
        \leq \bm{u}_i\left(s_{j}-\frac{\Delta t}{2}.v_{j}\right) 
\end{equation}
for all vehicles $i$ in sequence $q$ and for all vehicles $j$ in a conflicting group to vehicle $i$.

Finally, we enforce reachability by bounding the command velocities by the speed limit, $\bar v$, as well as the acceleration and breaking capabilities of each individual vehicle.

\begin{equation}
\label{eq:eq_cons_1}
    0\leq \bm{u}_i(t)\leq \bar v
\end{equation}
\begin{equation}
\label{eq:eq_cons_2}
    a^{min}_i \Delta t\leq \bm{u}_i(t)-v_i(t)\leq a^{max}_i \Delta t
\end{equation}

\noindent\textbf{Solving \optname:} \textit{The constraints (\ref{eq: long_linear}), (\ref{eq: lateral_linear}), (\ref{eq:eq_cons_1}), and (\ref{eq:eq_cons_2}) are provided to \optname with objective function given by Equation~(\ref{eq:J2}), which is solved numerically using Gurobi (version 9.1.2) \cite{gurobi}.}

The output of this objective function includes the computed target $\bm{u}_i$ command velocities for each CAV $i$ and the minimum objective value achieved for the sequence tested $q$. This information is saved for evaluation in the next step.

\noindent \textbf{Obtaining optimal sequence $q^*$ and command velocities $\bm{u}_i^*$}

Once all the optimization processes for each merging sequence $q$, generated in section \ref{subsec: priority_order} are completed, we analyze their output data to determine the optimal sequence $q^*$ and command velocities $\bm{u}_i^*$. We compare the objective values achieved for each tested sequence $q$, and the sequence providing the minimum objective value is deemed to be the optimal sequence $q^*$. The corresponding command velocities $\bm{u}_i$ are then deemed to be the optimal target command velocities $\bm{u}_i^*$. 

The optimal target $\bm{u}_i^*$ command velocities are then transmitted to each of the CAVs in the control zone. Note that we do not consider delays in communication and assume data is transmitted instantaneously via V2I communication. The low-level controller on-board each CAV then computes the acceleration or deceleration needed, to achieve the desired $\bm{u}_i$ command velocity within the control time ($\Delta t$) duration. 

Finally, a key feature of our approach is that it is robust to drift in the executed control. Our approach operates at a cycle frequency of $100$ ms after which new command velocities are obtained, updating the previous controls. Therefore, any deviation between optimal control and the executed control is erased during every $100$ ms update cycle.

\section{Experiments and Results}
\label{sec: experiments_and_results}

In this section, we evaluate the capabilities of our approach, study the impact of varying environmental parameters, and compare against other real-time methods.
\subsection{Experiment Setup}
We evaluate the performance of \modelNew in a multi-lane four-arm intersection simulation using the SUMO platform~\cite{SUMO2018} with the following parameters; \textit{Control zone length ($L_c$) = $150$m}, \textit{Speed limit ($\bar v$) = $20$m/s} \textit{Objective trade-off ($\lambda$) = $0.7$}, \textit{Safety margins: $M_{sr}$ = $2.0$ and $M_{sl}$ = $25.0$}. The high-level controller uses the TraCI traffic controller interface to communicate with the simulation. All simulations and optimization algorithms run on a personal computer with an Intel i7-8750H CPU and 32GB of RAM.
\vspace{-5pt}
\subsection{Evaluation Metrics and Baselines}
The key performance metrics used to evaluate control algorithms are \emph{throughput} (number of vehicles that can travel through the intersection per min), \emph{time-to-goal} (average time vehicles spend in the control zone) and \emph{fuel consumption}. We also add the metric of \emph{CO2 consumption} as carbon emissions are an important criterion in overall vehicle emissions. Furthermore, to showcase the performance in the presence of heterogeneous vehicles, we include separate results on \emph{time-to-goal of EVs} (Emergency Vehicles) and \emph{fuel consumption of trucks}. 
To showcase the properties and appropriately assess the performance
of \modelNew, we compare it with three other real-time capable baseline approaches using the same inputs.  
The first is an auction framework where the bidding strategy is based on the order in which agents arrive at the intersection, similar to the FIFO principle. To test the limits of our approach, we also compare against well-tuned traffic lights (light timing set to maximize performance), a well-established signaled intersection traffic management system. Additionally, we compare against another baseline which is a gap-based adaptive traffic light control strategy. This adaptive strategy allows traffic lights to dynamically adapt to current traffic conditions in the incoming lanes of an intersection by prolonging traffic phases whenever a continuous stream of traffic is detected. 
We refer to these baselines as \textsc{Stop-sign}, \textsc{Traffic-Light}, and \textsc{Actuated-Traffic-Light}. We refer to our method as \textsc{Coop-control}. Note that the Krauss vehicle following model was used in these comparisons.

\subsection{Performance analysis of \modelNew}

\begin{figure}[ht!]
\centering
   \begin{subfigure}[ht]{0.49\textwidth}
    \includegraphics[width=\textwidth]{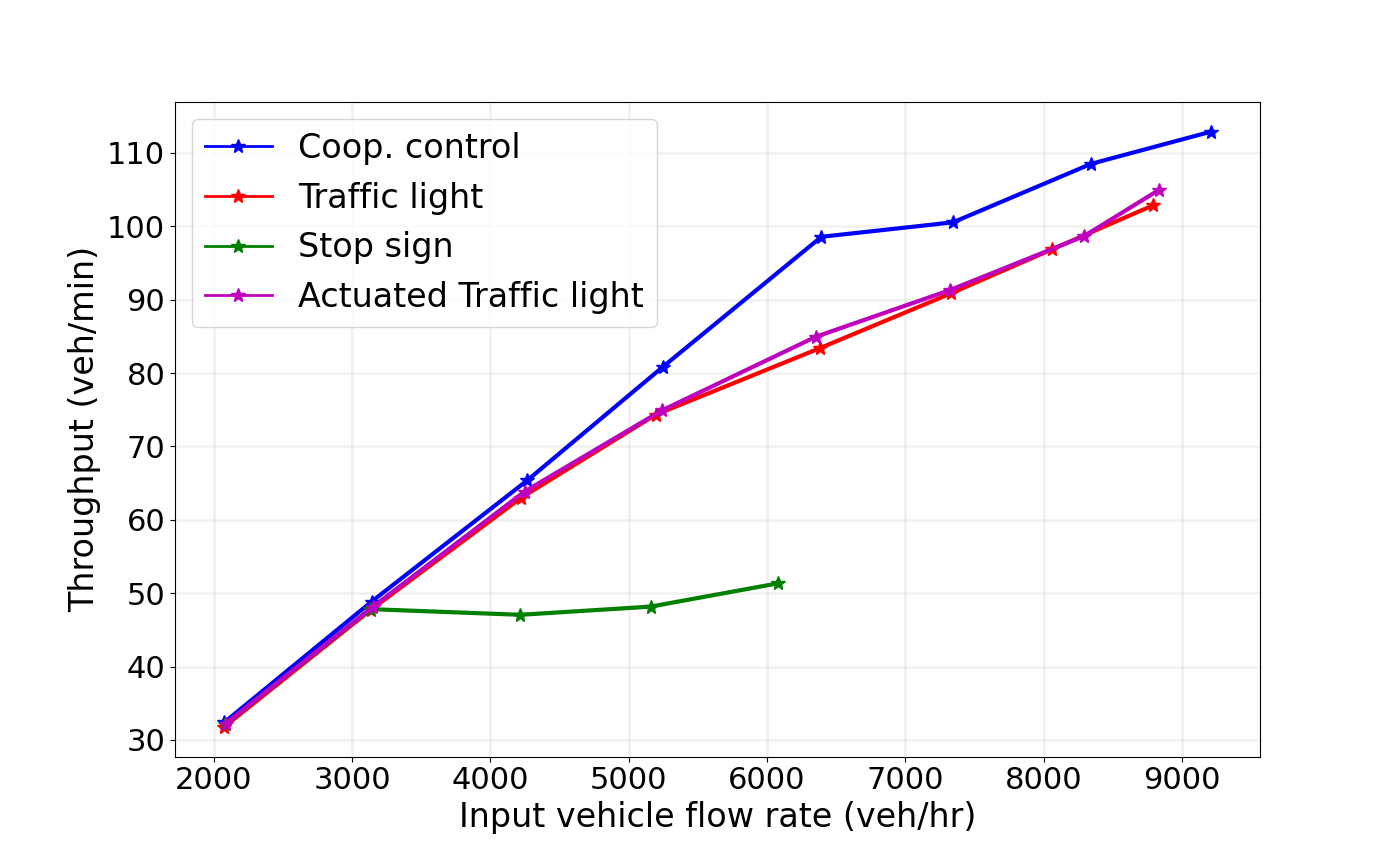}
    \caption{Throughput comparison}
    \label{fig: Throughput_go v2}
  \end{subfigure}
 \begin{subfigure}[ht]{0.49\textwidth}
    \includegraphics[width=\textwidth]{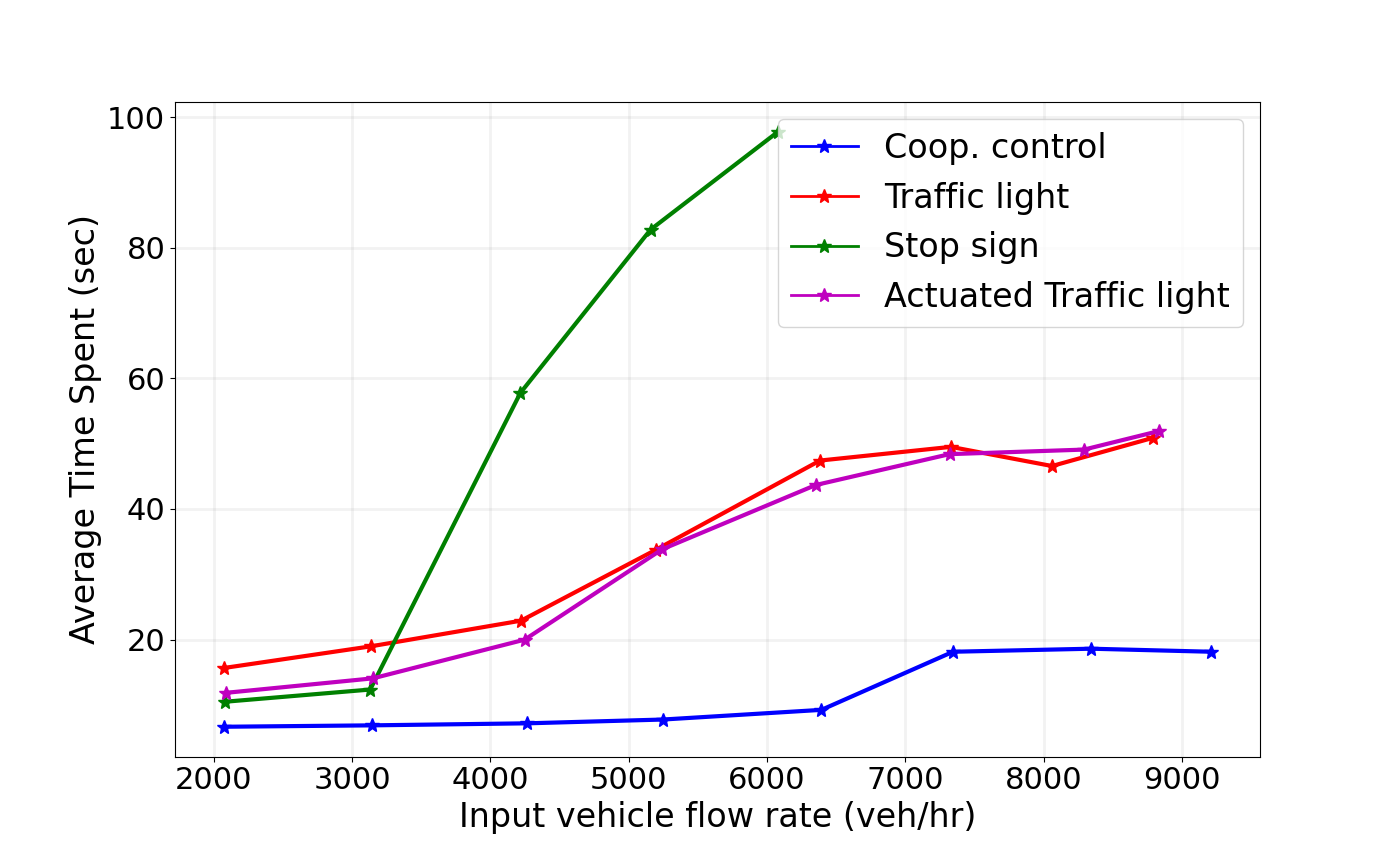}
    \caption{Time to goal comparison}
    \label{fig: TTG_go v2}
  \end{subfigure}
\caption{Performance in terms of throughput and time-to-goal in comparison to baselines, with increasing vehicle input flow rate.}
  \label{fig: throughput and time_spent v2} 
\end{figure} 

The results of these experiments for the two main traffic performance metrics of throughput and time spent are shown in figure \ref{fig: throughput and time_spent v2}. It is observed that our method (\textsc{Coop-control}) significantly improves performance relative to the traffic light-based methods with the relative improvement in performance increasing with increasing input flow rate. This follows our intuition that advanced multi-agent control methods provide a significant advantage when congestion is high. At low congestion levels, even simpler methods are often sufficient. The \textsc{Stop-sign} method's performance however is very poor even at moderately high input flow rates. In fact beyond a certain vehicle flow rate, due to the low throughput in the \textsc{Stop-sign} method, the input lanes become fully packed and it is impossible to increase the input flow rate any further. Our proposed \textsc{Coop-control} method is observed to provide improvements to \emph{throughput} of around $25\%$ over the \textsc{Traffic-Light} and \textsc{Actuated-Traffic-Light} methods and around $137\%$ improvement over the \textsc{Stop-sign} method as shown in figure \ref{fig: Throughput_go v2}. Additionally, as shown in figure \ref{fig: TTG_go v2}, our \textsc{Coop-control} method improves the average time spent by vehicles (\emph{time-to-goal}) by around $70\%$ over the \textsc{Traffic-Light} and \textsc{Actuated-Traffic-Light} methods and around $83\%$ over the \textsc{Stop-sign} method. 

Here, we find that both the timed \textsc{Traffic-Light} and \textsc{Actuated-Traffic-Light} methods offer similar performance which shows how well the timing on the \textsc{Traffic-Light} method has been tuned to match the rate of vehicle inflow on each of the intersections input arms. 

\begin{figure}[ht!]
\centering
   \begin{subfigure}[ht]{0.49\textwidth}
    \includegraphics[width=\textwidth]{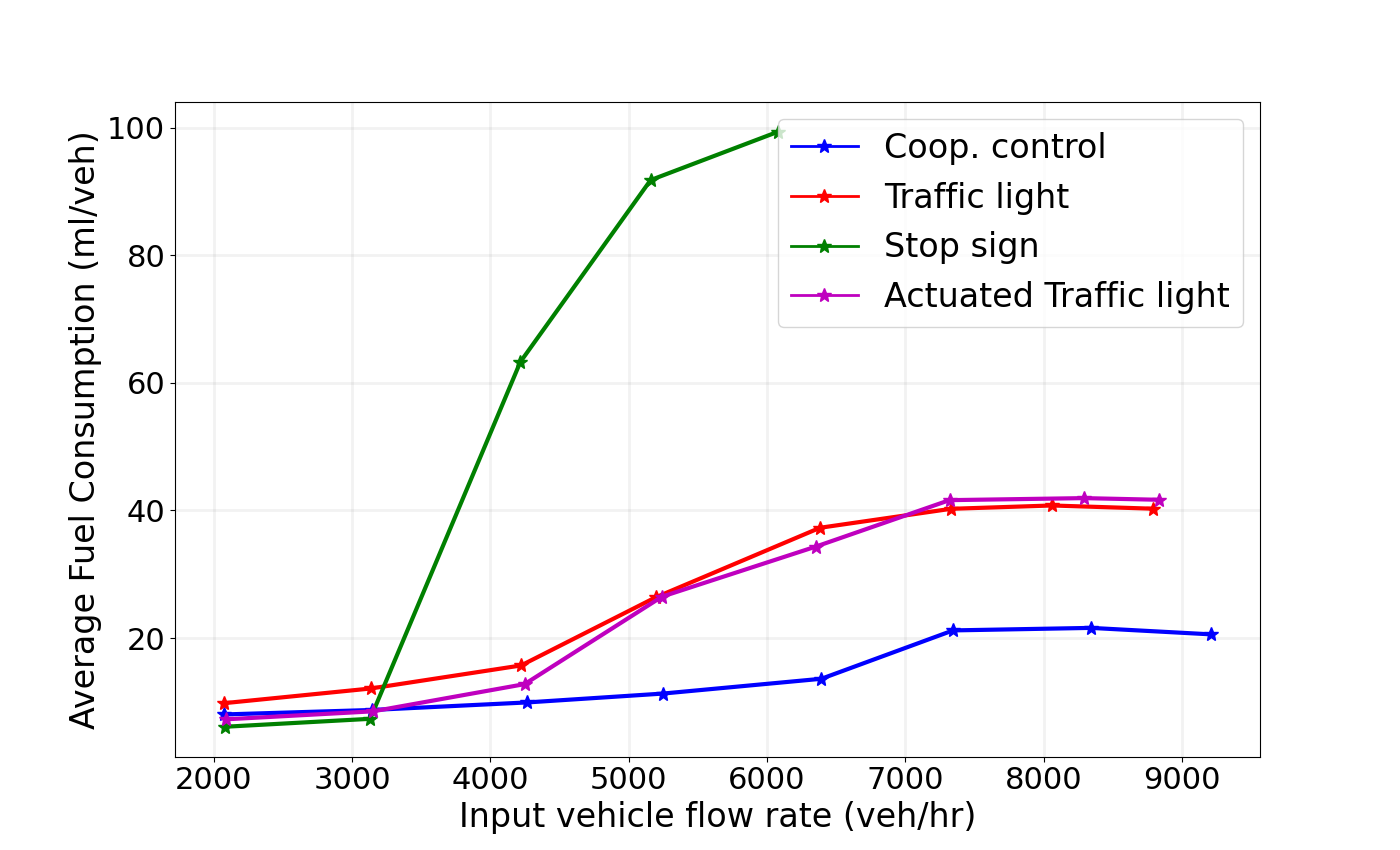}
    \caption{Fuel Consumption comparison}
    \label{fig: fuel_go v2}
  \end{subfigure}
 \begin{subfigure}[ht]{0.49\textwidth}
    \includegraphics[width=\textwidth]{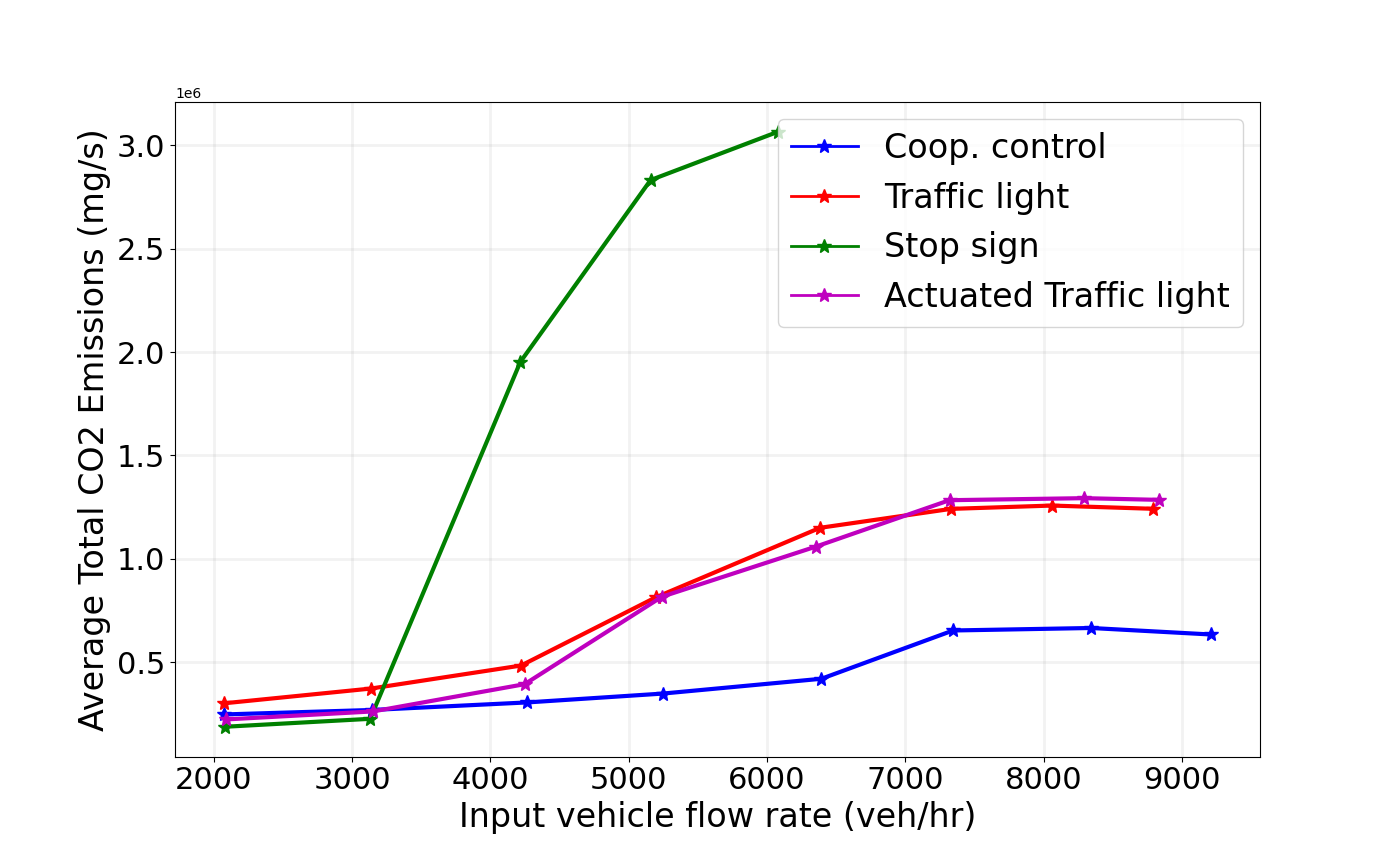}
    \caption{CO2 emissions comparison}
    \label{fig: CO2_go v2}
  \end{subfigure}
\caption{Performance in terms of fuel consumption and CO2 emissions in comparison to baselines, with increasing vehicle input flow rate.}
  \label{fig: fuel and CO2 v2} 
\end{figure}  

The key indicators we consider for efficiency in terms of environmental factors are total fuel consumption and CO2 emissions. The performance of our \textsc{Coop-control} method for these metrics is shown in figure \ref{fig: fuel and CO2 v2}, where we observe significant improvements over baseline methods. Once again we observe that the level of improvement increases at higher input vehicle flow rates. 
In figure \ref{fig: fuel_go v2}, we observe a $50\%$ and $80\%$ reduction in fuel consumption relative to the \textsc{Traffic-Light} and \textsc{Stop-sign} methods respectively. The same trend is observed for CO2 emissions in figure \ref{fig: CO2_go v2}, where we see a reduction of $50\%$ and $80\%$ relative to the \textsc{Traffic-Light} and \textsc{Stop-sign} methods respectively. This ability of our method to reduce fuel usage and CO2 emissions by such significant margins could have a major impact on overall environmental sustainability. In our \textsc{Coop-control} method fewer vehicles are forced to unnecessarily slow down and speed up which equates to less fuel usage to cover the same distance. Additionally, when acceleration is necessary the control algorithm attempts to minimize the amount of acceleration required.

\begin{figure}[ht!]
\centering
   \begin{subfigure}[ht]{0.49\textwidth}
    \includegraphics[width=\textwidth]{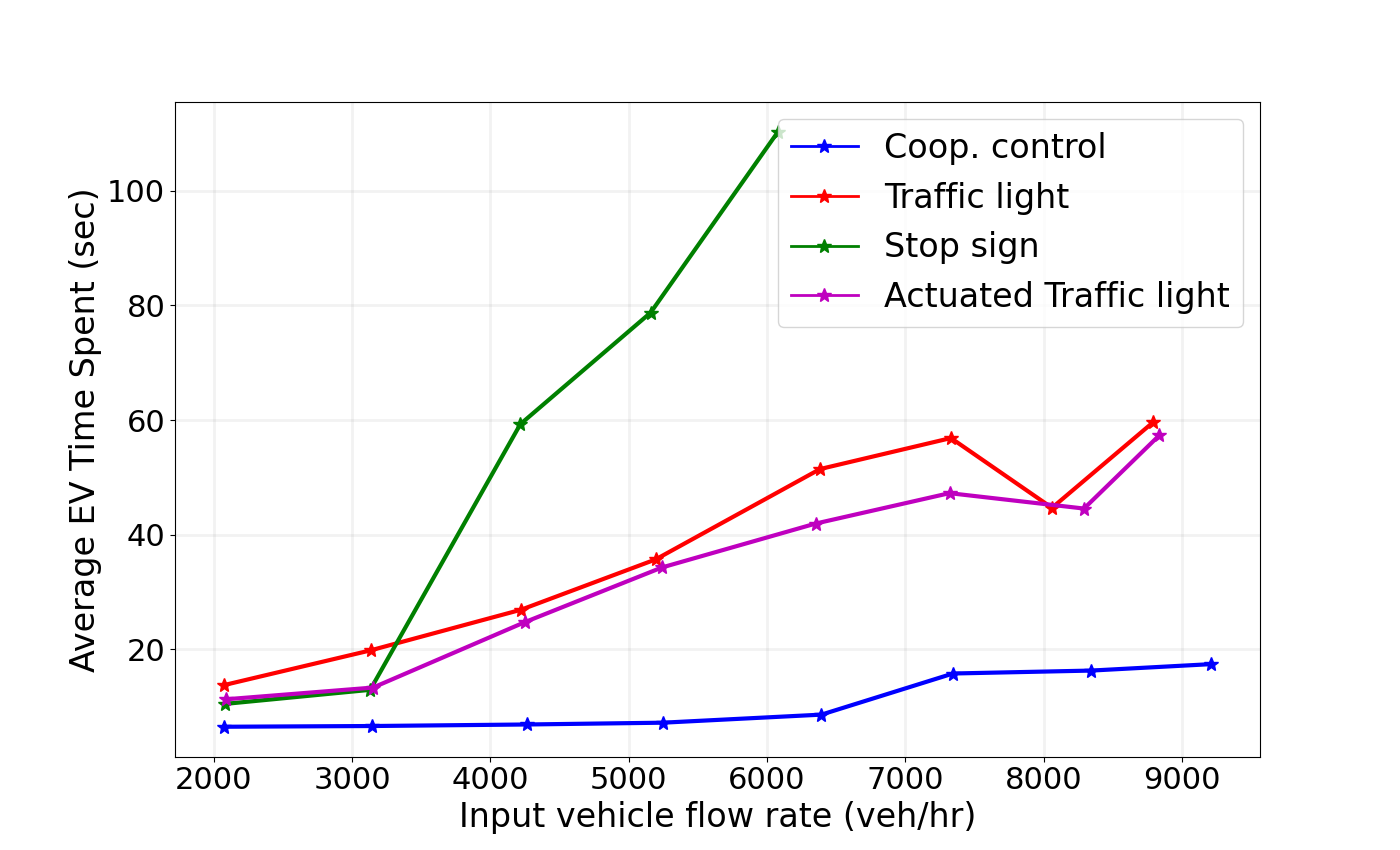}
    \caption{EV Time to goal comparison}
    \label{fig: EV_TTG_go v2}
  \end{subfigure}
 \begin{subfigure}[ht]{0.49\textwidth}
    \includegraphics[width=\textwidth]{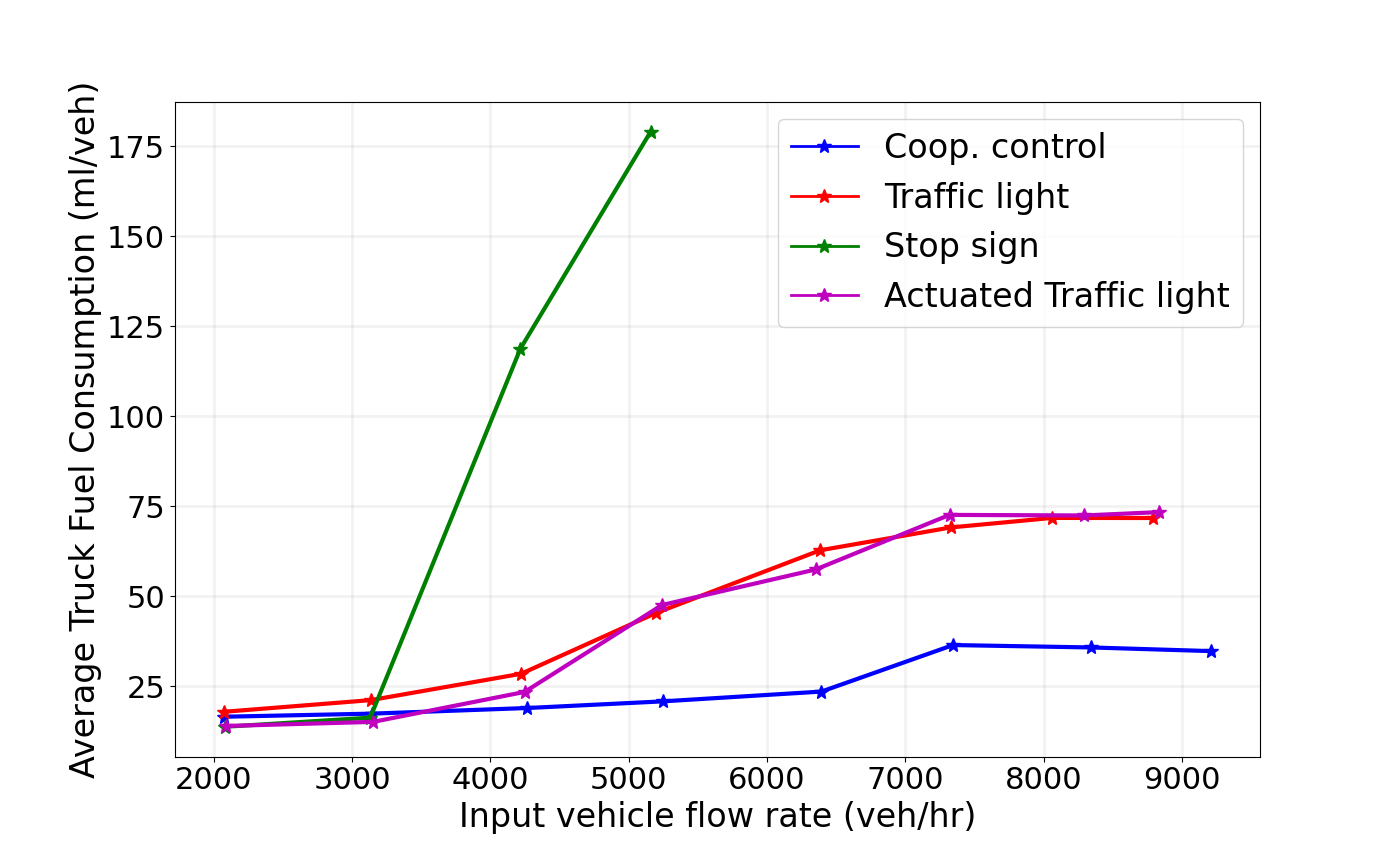}
    \caption{Truck Fuel Consumption comparison}
    \label{fig: truck_fuel_go v2}
  \end{subfigure}
\caption{Performance in terms of emergency vehicle time-to-goal and truck fuel consumption in comparison to baselines, with increasing vehicle input flow rate.}
  \label{fig: EV TTG and Truck fuel v2} 
\end{figure}  

To showcase our method's ability to handle heterogeneous traffic we also show a comparison of the performance impact for the two major edge cases of emergency vehicles minimizing their \emph{time-to-goal of EVs} and the minimizing of \emph{fuel consumption of trucks}. In a system where each vehicle's actions can affect the other neighboring vehicles, these two goals of minimizing time and fuel consumption are generally opposing each other. However as shown in figure \ref{fig: EV TTG and Truck fuel v2}, our method can successfully minimize both metrics in order to satisfy the individual needs of each vehicle class. Over the standard benchmark of the \textsc{Traffic-Light} method, our \textsc{Coop-control} algorithm is shown to provide reductions of around $75\%$ and $54\%$ for \emph{time-to-goal of EVs} and \emph{fuel consumption of trucks} respectively. When compared to figures \ref{fig: TTG_go v2} and \ref{fig: fuel_go v2}, we see that EVs face fewer relative delays than the rest of the traffic and trucks relatively consume less fuel than the rest of the traffic. 

\subsubsection{Impact of unbalanced traffic inflow on \modelNew}
A major benefit of our method is its ability to adapt to dynamic changes to the input vehicle flow rate on each of the input arms of the intersection. When all the arms have equal input flow the system is symmetrical and we consider the input flow to be balanced. However, when the input flow rate on one or more of the arms is noticeably higher than the others, we consider the input flow to be unbalanced. To evaluate our method's robustness to unbalanced traffic inflow we conduct experiments where the amount of traffic originating in the horizontal lanes (from West to East or vice versa) is gradually increased relative to the traffic originating in the vertical lanes (from North to South or vice versa). We then plot the same performance metrics against the increasing ratio of horizontal lane to vertical lane traffic. Note that in these experiments, to allow for accurate comparisons, the flow rates on each arm are adjusted so that the overall total input flow rate remains fixed at $5200~veh/hr$.

\begin{figure}[ht!]
\centering
   \begin{subfigure}[ht]{0.49\textwidth}
    \includegraphics[width=\textwidth]{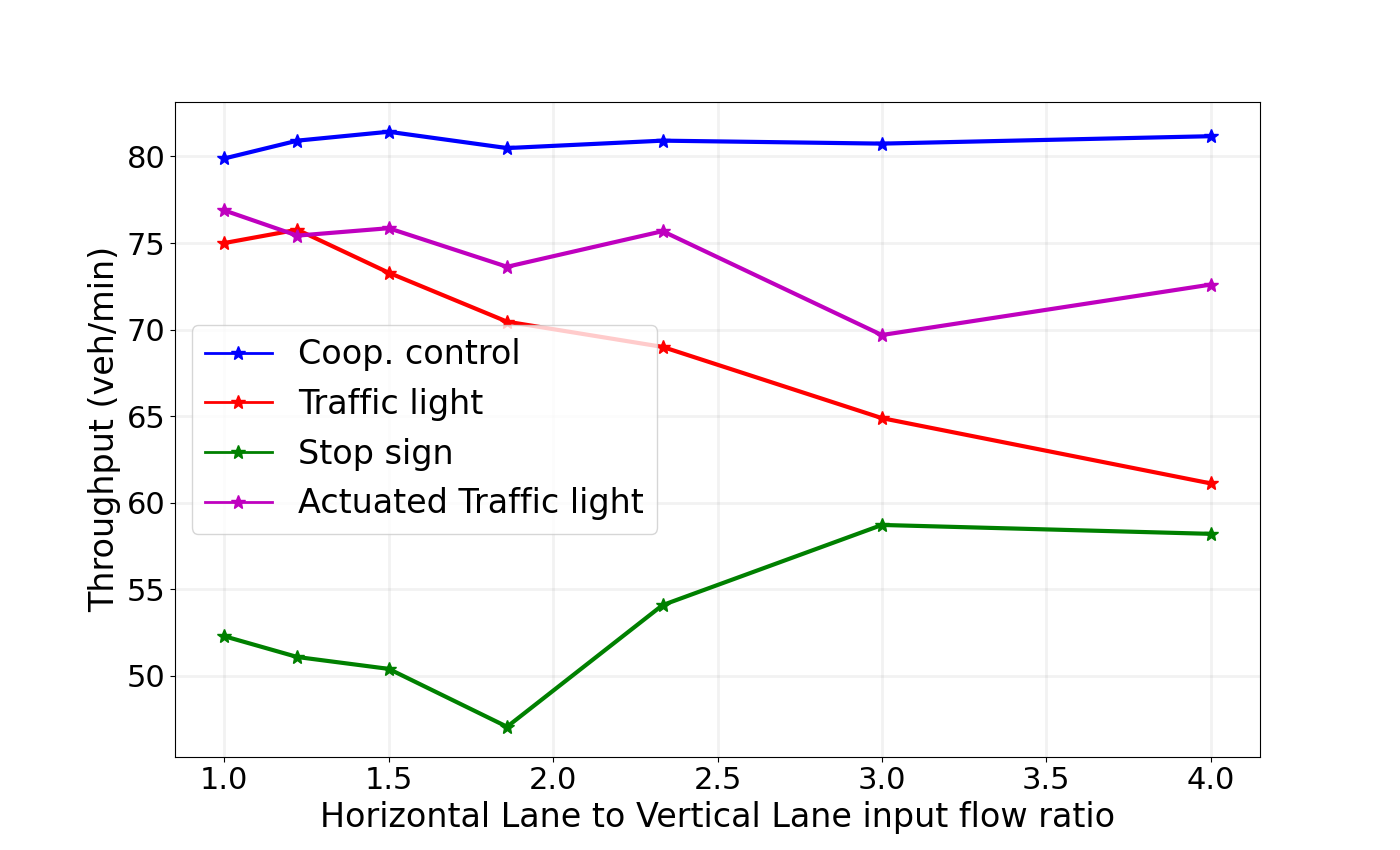}
    \caption{Throughput comparison with unbalanced inflow}
    \label{fig: Throughput_go_un v2}
  \end{subfigure}
 \begin{subfigure}[ht]{0.49\textwidth}
    \includegraphics[width=\textwidth]{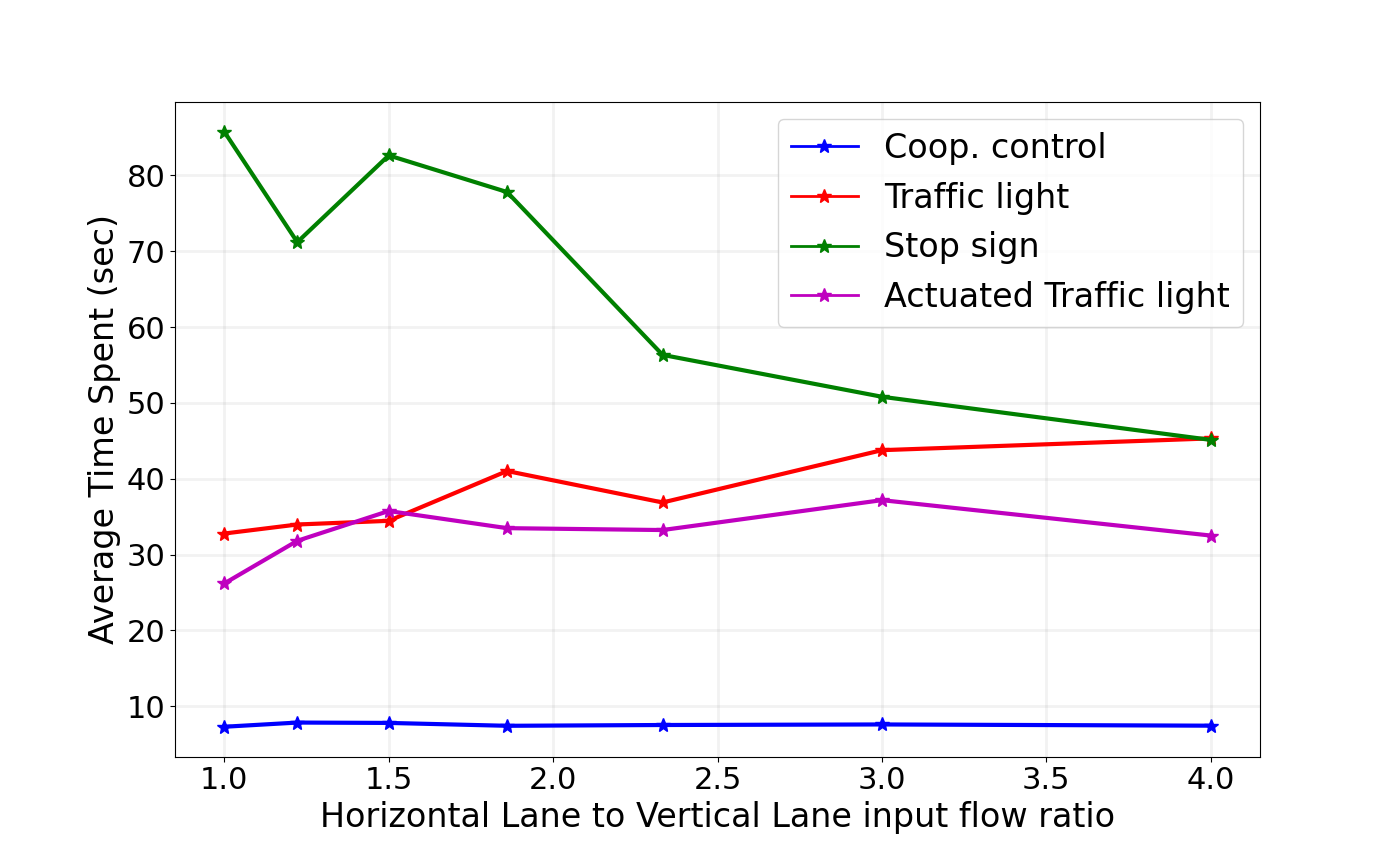}
    \caption{Time to goal comparison with unbalanced inflow}
    \label{fig: TTG_go_un v2}
  \end{subfigure}
\caption{Performance in terms of throughput and time-to-goal in comparison to baselines, with unbalanced vehicle input flow rate. Total flow rate fixed at $5200~veh/hr$.}
  \label{fig: throughput and time_spent_un v2} 
\end{figure}

The effect of unbalanced traffic inflow on throughput and time spent is presented in figure \ref{fig: throughput and time_spent_un v2}. Our proposed \textsc{Coop-control} method is unperturbed by changes in the traffic inflow balance. Both the throughput and time spent remain steady under increasingly unbalanced input flow when our proposed \textsc{Coop-control} method is used. The timed \textsc{Traffic-Light} approach is the most affected method, with performance dropping by about $20\%$ and $45\%$ in throughput and time spent when the balance between horizontal and vertical lanes changes from $1:1$ to $4:1$. This is expected since the \textsc{Traffic-Light} method timing is tuned to handle a symmetric load of traffic. As such, in symmetric inflow conditions ($1:1$) the \textsc{Traffic-Light} performs equivalent to the \textsc{Actuated-Traffic-Light} method which is designed to adapt to possible changes in traffic symmetry. Therefore we see a lesser performance drop of around $5\%$ in throughput in the \textsc{Actuated-Traffic-Light} method when the traffic becomes more unbalanced. While the \textsc{Stop-sign} method generally performs much worse than its competitors, we observe an interesting phenomenon of improved performance when the traffic flow becomes increasingly unbalanced. The \textsc{Stop-sign} method is inherently designed to handle unbalanced traffic as and when required. When the traffic is more unbalanced, then the traffic moving through the intersection is less diverse in terms of possible directions it takes. As such the vehicles moving in the majority-use directions can pass the intersection with less waiting especially when they belong to non-conflicting groups. This is the primary reason the stop sign method performs better when traffic is less balanced with around a $10\%$ improvement in throughput in more unbalanced traffic. 

\begin{figure}[ht!]
\centering
   \begin{subfigure}[ht]{0.49\textwidth}
    \includegraphics[width=\textwidth]{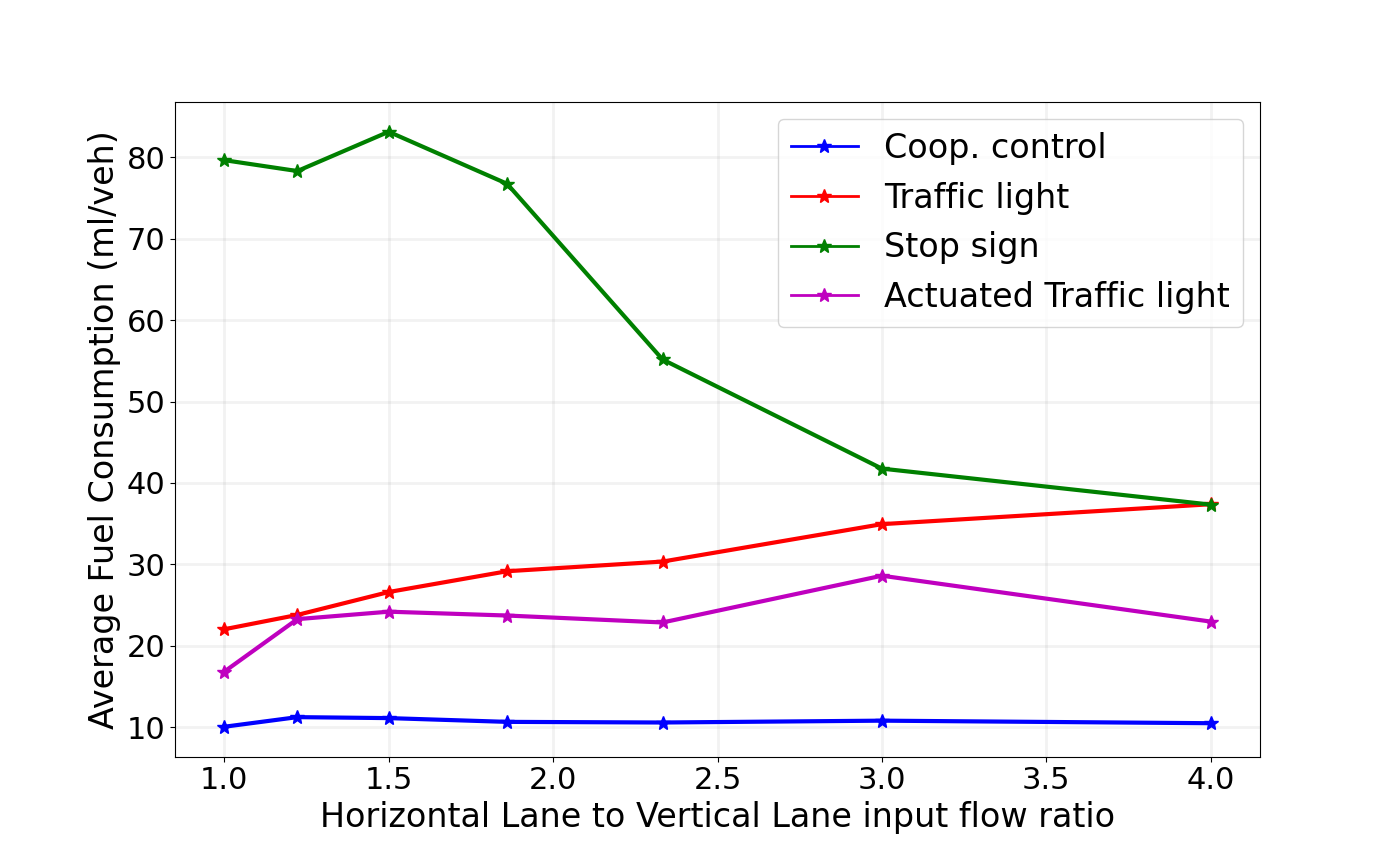}
    \caption{Fuel Consumption comparison with unbalanced inflow}
    \label{fig: fuel_go_un v2}
  \end{subfigure}
 \begin{subfigure}[ht]{0.49\textwidth}
    \includegraphics[width=\textwidth]{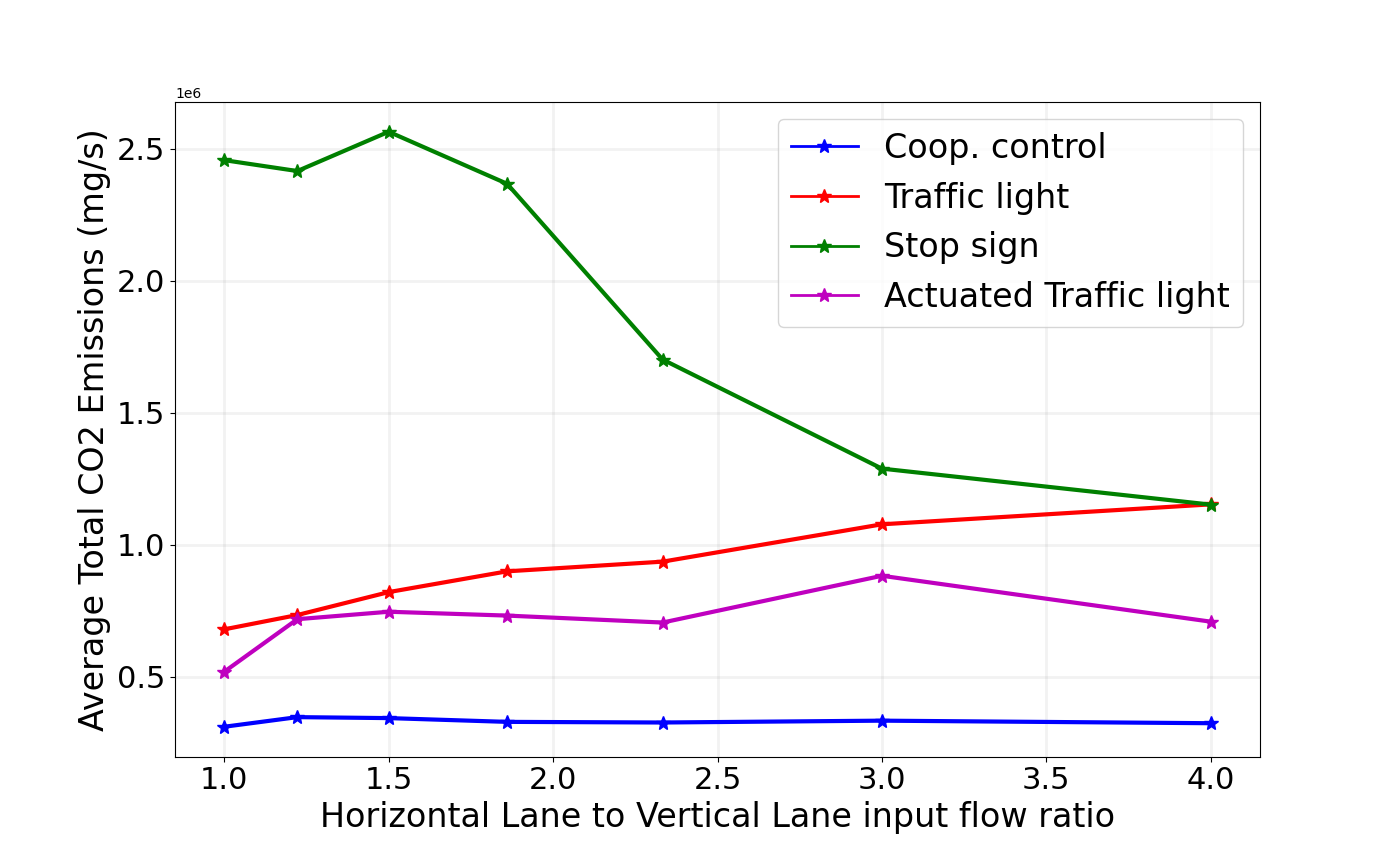}
    \caption{CO2 emissions comparison with unbalanced inflow}
    \label{fig: CO2_go_un v2}
  \end{subfigure}
\caption{Performance in terms of fuel consumption and CO2 emissions in comparison to baselines, with unbalanced vehicle input flow rate. Total flow rate fixed at $5200~veh/hr$.}
  \label{fig: fuel and CO2_un v2} 
\end{figure} 

The same observations hold in terms of fuel consumption and CO2 emissions as shown in figure \ref{fig: fuel and CO2_un v2}. Our \textsc{Coop-control} method which offers the best performance also maintains this performance irrespective of the balance of traffic flow. The \textsc{Traffic-Light} shows the highest performance degradation while the \textsc{Actuated-Traffic-Light} method lies somewhat in between. As in the earlier case the \textsc{Stop-sign} method performance improves in unbalanced traffic and is actually capable of achieving fuel consumption related performance equivalent to the \textsc{Traffic-Light} method in the extreme unbalanced traffic condition. 

\begin{figure}[ht!]
\centering
   \begin{subfigure}[ht]{0.49\textwidth}
    \includegraphics[width=\textwidth]{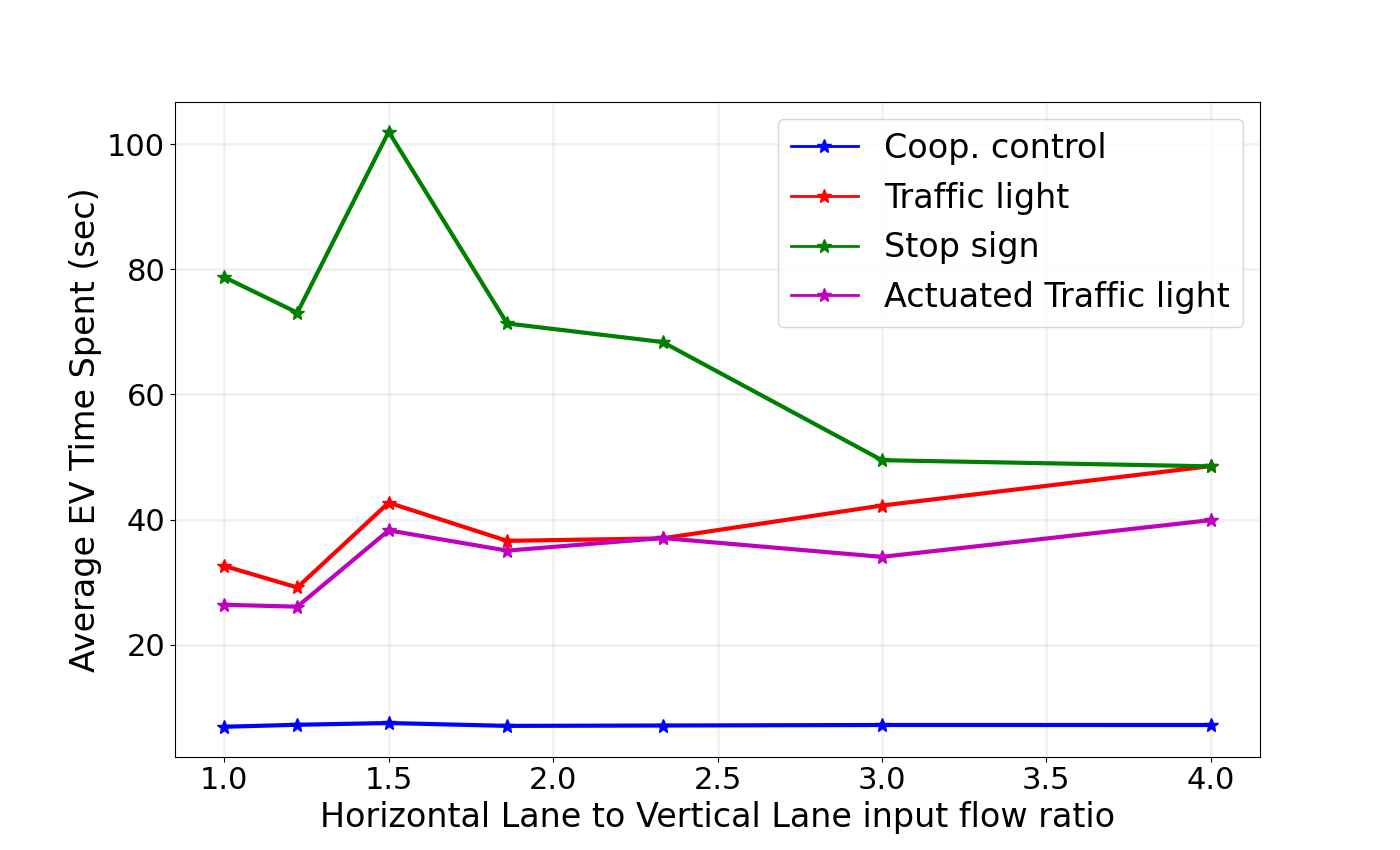}
    \caption{EV Time to goal comparison with unbalanced inflow}
    \label{fig: EV_TTG_go_un v2}
  \end{subfigure}
 \begin{subfigure}[ht]{0.49\textwidth}
    \includegraphics[width=\textwidth]{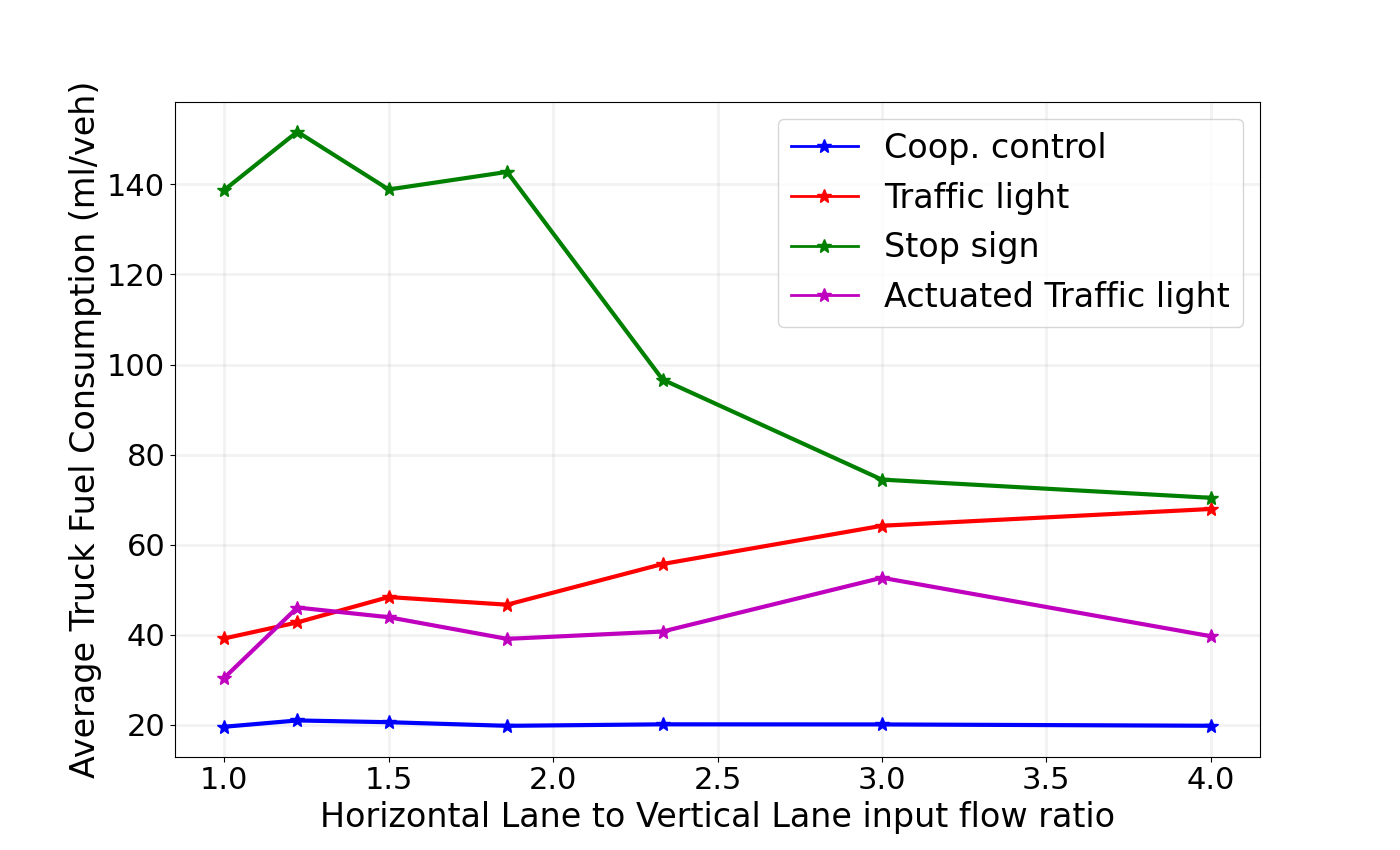}
    \caption{Truck Fuel Consumption comparison with unbalanced inflow}
    \label{fig: truck_fuel_go_un v2}
  \end{subfigure}
\caption{Performance in terms of emergency vehicle time-to-goal and truck fuel consumption in comparison to baselines, with unbalanced vehicle input flow rate. Total flow rate fixed at $5200~veh/hr$.}
  \label{fig: EV TTG and Truck fuel_un v2} 
\end{figure}  

Finally, as shown in figure \ref{fig: EV TTG and Truck fuel_un v2}, we find that these same observed trends hold in the specific heterogeneous vehicle cases of emergency vehicle time to goal and truck fuel consumption. 

Overall we identify that our \textsc{Coop-control} method demonstrates high levels of robustness to changes in the balance of input traffic flow. It not only exhibits the highest performance among all other competing methods but is also able to maintain this level of performance irrespective of the level of symmetry in the input traffic flow.
\section{Conclusion, Limitations, and Future Work}
\label{sec: conclusion}

We introduce \modelNew, our innovative hybrid solution for cooperative control at dynamic, multi-lane, unsignalized intersections. Our algorithm combines an auction mechanism for generating a priority entrance sequence with an optimization-based trajectory planner to determine optimal velocity commands aligning with the sequence. This unique approach enables real-time computation capabilities in high-density multi-lane traffic while ensuring efficiency, safety, and fairness. Through extensive testing on the SUMO platform, we validate that \modelNew enhances throughput by approximately $25\%$, reduces time spent by $70\%$, and decreases fuel consumption by $50\%$ compared to auction-based and signaled approaches involving traffic lights and stop signs. Additionally, unlike the traffic light baselines which had performance drops, \modelNew is unaffected by unbalanced traffic flow environments. Notably, our results demonstrate that \modelNew operates at remarkable real-time speeds ($<10$ milliseconds), surpassing prior methods by at least $100\times$.


There are several avenues for further research in this domain. One area of focus would involve investigating various types of auctions to enhance flexibility in managing intersection control. Different priority orderings of vehicles can be more efficient for different scenarios. Additionally, it would be valuable to explore the effects of imperfect communication on the algorithm's performance. Real-world systems encounter disruptions and unstable communication lines, so replicating this would help in creating a more robust solution. Furthermore, reproducing the information processing capabilities and delays of the sensors onboard the vehicles would generate a more realistic simulation. 


{\footnotesize \bibliography{refs}}
\bibliographystyle{IEEEtran}

\end{document}